\documentclass{article} 
\usepackage{iclr2025_conference,times}

\usepackage{amsmath,amsfonts,bm}

\def\eqref#1{equation~\ref{#1}}

\def\1{\bm{1}}

\DeclareMathAlphabet{\mathsfit}{\encodingdefault}{\sfdefault}{m}{sl}
\SetMathAlphabet{\mathsfit}{bold}{\encodingdefault}{\sfdefault}{bx}{n}

\newcommand{\E}{\mathbb{E}}

\usepackage{hyperref}
\usepackage{url}

\usepackage[table]{xcolor}
\usepackage{soul}
\usepackage{booktabs}
\usepackage{graphicx}
\usepackage{array}
\usepackage{amsmath}
\usepackage{amssymb}

\definecolor{topcolor}{HTML}{FFF2CC}
\definecolor{secondcolor}{HTML}{D9EAD3}

\usepackage{booktabs}  
\usepackage{tabularx}  
\usepackage{array}     

\usepackage{amsmath,amssymb,amsthm}
\newtheorem{proposition}{Proposition}
\newtheorem{corollary}{Corollary}

\newtheorem{remark}{Remark}

\usepackage{booktabs}
\usepackage{tabularx}
\usepackage{array}

\usepackage{algorithm}
\usepackage{algpseudocode}
\usepackage{xcolor}
\usepackage{amsmath,amssymb}

\title{OPID: On-Policy Skill Distillation for Agentic Reinforcement Learning}

\author{\textbf{Shuo Yang}$^{1}$\thanks{Equal Contribution}~, \textbf{Jinyang Wu}$^{1}$\footnotemark[1]~~\thanks{Project Leader}~~,
\textbf{Zhengxi Lu}$^{2}$, \textbf{Yuhao Shen}$^{2}$, \textbf{Fan Zhang}$^{3}$,
\textbf{Lang Feng}$^{4}$,\\
\textbf{Shuai Zhang}$^{1}$,
\textbf{Haoran Luo}$^{4}$,
\textbf{Zheng Lian}$^{5}$,
\textbf{Zhengqi Wen}$^{1}$,
\textbf{Jianhua Tao}$^{1}$\\
\\
$^1$Tsinghua University \qquad$^2$Zhejiang University \qquad$^3$The Chinese University of Hong Kong\\
$^4$Nanyang Technological University \qquad$^5$Tongji University\\
\texttt{Corresponding to: wu-jy23@mails.tsinghua.edu.cn} \\
\begin{tabular}{@{}ll@{}}
\end{tabular}}

\usepackage{colortbl}
\definecolor{gain}{RGB}{0,128,0}
\definecolor{purple}{RGB}{216, 110, 204}

\usepackage{hyperref}
\hypersetup{
    colorlinks=true,      
    citecolor=purple,       
}

\iclrfinalcopy 
\begin{document}

\maketitle

\begin{abstract}
Outcome-based reinforcement learning provides a stable optimization backbone for language agents, but its sparse trajectory-level rewards provide little guidance on which intermediate decisions should be reinforced or suppressed. On-policy self-distillation offers dense token-level supervision, yet existing skill-conditioned variants often rely on external skill memories or retrieved privileged context, which are costly to maintain and can be mismatched with the state distribution induced by the current policy in multi-turn interaction. We propose \textbf{OPID} (\textbf{O}n-\textbf{P}olicy Sk\textbf{i}ll \textbf{D}istillation), a framework that extracts skill supervision directly from completed on-policy trajectories. OPID represents trajectory hindsight as hierarchical skills: episode-level skills capture global workflows or failure-avoidance rules, while step-level skills capture local decision knowledge at critical timesteps. A critical-first routing mechanism uses step-level skills when critical decisions are identified and falls back to episode-level skills as default guidance otherwise. The selected skill is injected into the interaction history, allowing the old policy to re-score the same sampled response under both original and skill-augmented contexts. The resulting log-probability shift yields a token-level self-distillation advantage, which is combined with the outcome advantage for policy optimization. OPID thus preserves RL as the primary training objective while introducing dense, distribution-matched hindsight supervision. Experiments on ALFWorld, WebShop and Search-based QA demonstrate that OPID generally improves agent performance, sample efficiency, and robustness over outcome-only RL and existing skill-distillation baselines. Our code is available at \url{https://github.com/jinyangwu/OPID/tree/main}.
\end{abstract}

\begin{figure}[h]
    \centering
    \includegraphics[width=0.99\linewidth]{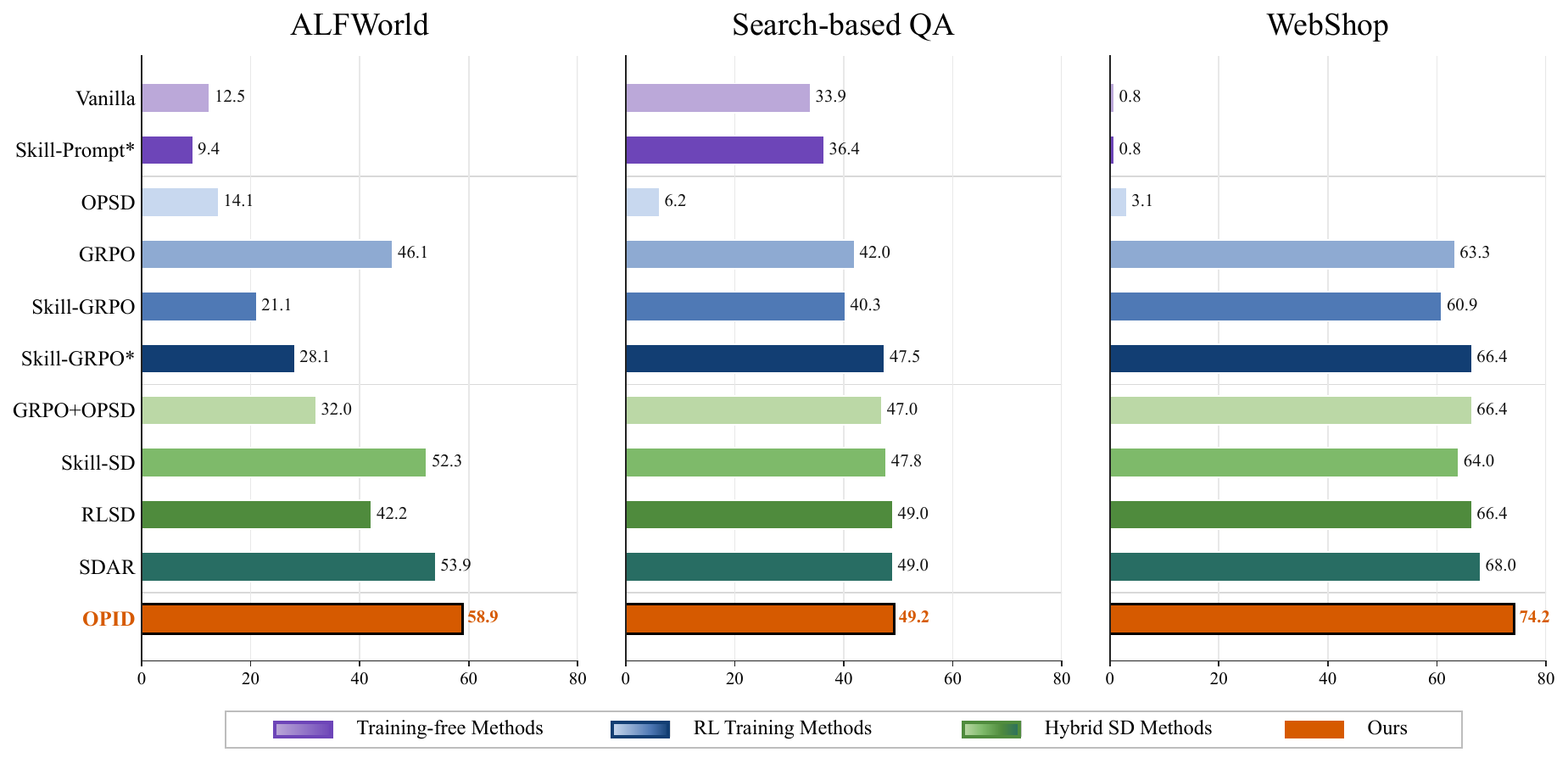}
    \caption{
        \textbf{Overall performance comparison.}
        We compare OPID with training-free prompting methods, outcome-only RL, and skill-distillation baselines on ALFWorld, Search-based QA, and WebShop. OPID achieves the strongest average performance on ALFWorld and WebShop while remaining competitive on Search-based QA.
    }
    \label{fig:overall_performance}
\end{figure}

\section{Introduction}
Large language models (LLMs) are increasingly deployed as interactive agents that operate over long horizons, invoke tools, navigate environments, and adapt their behavior through multi-turn observations~\citep{jimenez2024swebench,luo2025large,wu2026spark,lu2026sdar}. Unlike single-turn reasoning, agentic tasks require sequential decisions whose consequences may only become visible after many interaction steps. This setting spans embodied household environments, web navigation, search-augmented reasoning, and software engineering agents~\citep{shridhar2020alfworld,yao2022webshop,jin2025searchr1,jimenez2023swebench}. Reinforcement learning (RL) has become a natural post-training paradigm for such agents, since it directly optimizes policies using task-level feedback from environments or verifiers. In particular, outcome-based methods such as GRPO~\citep{shao2024deepseekmath} provide a stable critic-free optimization backbone for on-policy rollouts.

Despite its effectiveness, outcome-based agentic RL offers only coarse supervision~\citep{zhang2025landscape}. Environment rewards are typically sparse, delayed, and high-variance: a terminal reward can indicate whether a trajectory succeeds, but not which intermediate decisions caused the outcome. This limitation is especially severe in long-horizon interaction~\citep{chen2026iterresearch,xu2026odysseyarena}, where a single early mistake may derail the episode, repeated invalid actions may accumulate over time, and the effect of a local decision may only be observed several turns later. As a result, purely outcome-driven optimization provides stable task-level pressure but lacks fine-grained decision-level credit assignment.

On-policy distillation and self-distillation provide complementary supervision. Rather than relying solely on trajectory-level rewards, on-policy distillation trains models on their own sampled outputs while using auxiliary teacher signals to induce token-level guidance~\citep{gu2023minillm,agarwal2024onpolicy}. Recent self-distillation methods remove the need for a separate teacher by comparing the same policy under different contexts, such as a standard student branch and a privileged teacher branch~\citep{zhao2026opsd,he2026sdzero}. In agentic RL, this suggests a natural decomposition: RL remains the primary optimization backbone, while self-distillation supplies dense token-level shaping signals. Recent work such as SDAR follows this principle by treating self-distillation as a controlled auxiliary objective for multi-turn agents~\citep{lu2026sdar}.

A particularly promising form of privileged context is a natural-language skill. Skill-conditioned self-distillation augments the teacher branch with procedural knowledge, such as subgoal decompositions, action templates, or behavioral rules, and distills the resulting token-level preferences into the policy~\citep{lu2026skill0,wang2026skillsd,lu2026sdar}. However, existing skill-based methods typically rely on external skill libraries, retrieved skill files, or maintained skill memories. This design raises two challenges. First, skill memories require non-trivial maintenance, including skill insertion, refinement, deletion, and retrieval. Second, retrieved skills may be mismatched with the state distribution induced by the current policy. Such mismatch is particularly problematic for multi-turn agents, where small deviations from the assumed trajectory can lead to state drift and make an otherwise useful skill unreliable.

Based on this observation, we propose \textbf{OPID} (\textbf{O}n-\textbf{P}olicy Sk\textbf{i}ll \textbf{D}istillation), a framework that extracts hindsight skills from completed on-policy trajectories and distills their behavioral effects back into the policy. OPID abstracts each trajectory into two complementary levels of natural-language skills: \emph{episode-level skills}, which summarize trajectory-wide workflows or failure-avoidance rules, and \emph{step-level skills}, which capture state-conditioned guidance at critical timesteps. This hierarchy reflects a granularity trade-off in long-horizon decision making. Episode-level skills are broad and stable but may be too coarse for pivotal states, whereas step-level skills are precise but sparse and state-specific. OPID addresses this trade-off with \emph{critical-first skill routing}: it uses step-level skills at identified critical timesteps and falls back to episode-level skills otherwise. The routed skill is injected into the agent's interaction history, allowing the old policy to re-score the same on-policy response under both original and skill-augmented contexts. The induced token-level log-probability shift forms a skill-based self-distillation advantage, which is combined with the episode advantage for policy optimization. OPID therefore preserves outcome-based RL as the primary objective while introducing dense, on-policy hindsight supervision. At inference time, OPID requires no analyzer, external skill retrieval, or privileged context.

We evaluate OPID on ALFWorld~\citep{shridhar2020alfworld}, 
WebShop~\citep{yao2022webshop}, and Search-based QA~\citep{jin2025searchr1} with models at different scales. Across these settings, OPID improves long-horizon agent performance over outcome-only RL and skill-distillation baselines. These results suggest that completed on-policy trajectories provide a useful source of distribution-matched hindsight supervision, enabling the policy to internalize trajectory-derived skills without relying on external skill libraries or retrieved privileged context at inference time.

Taken together, our work makes the following contributions:
\begin{itemize}
    \item We propose \textbf{on-policy hindsight skill extraction}, which treats completed trajectories sampled by the current policy as a distribution-matched source of skill supervision, avoiding the need for external skill libraries or off-policy retrieval.

    \item We introduce \textbf{hierarchical hindsight skills with critical-first routing}, where episode-level skills capture global workflows or failure-avoidance rules, step-level skills 
    capture critical local decisions, and routing selects the most specific available skill for each trajectory step.

    \item We integrate \textbf{skill-based self-distillation} into agentic RL, converting routed hindsight skills into dense token-level shaping signals while preserving outcome reward optimization as the primary training objective.

    \item We empirically validate OPID on long-horizon agentic benchmarks, showing consistent improvements over outcome-only RL and skill-distillation baselines, along with better sample efficiency and reduced repetitive or invalid behaviors.
\end{itemize}

\section{Related Work}\label{sec:related_work}
\paragraph{Reinforcement learning for agentic LLMs.}
Large language models are increasingly trained as interactive agents that operate over long horizons, invoke tools, and receive feedback from environments or verifiers~\citep{shridhar2020alfworld,yao2022webshop,jin2025searchr1,jimenez2023swebench,wu2026atlas}. Reinforcement learning has therefore become a natural post-training paradigm, with outcome-based methods such as GRPO providing a stable critic-free objective for on-policy rollouts~\citep{shao2024deepseekmath}. However, agentic environments typically provide sparse and delayed rewards. A terminal outcome can indicate whether a trajectory succeeds, but it does not identify which intermediate decisions caused success or failure. OPID targets this missing credit-assignment signal: it keeps outcome-based RL as the optimization backbone, but augments it with dense decision-level supervision extracted from the policy's own completed trajectories.

\paragraph{On-policy self-distillation.}
On-policy distillation trains a model from its own sampled outputs while using auxiliary teacher signals to provide token-level learning targets~\citep{agarwal2024onpolicy,gu2023minillm}. Recent self-distillation methods further remove the need for a separate teacher by comparing the same policy under different contexts or feedback conditions~\citep{zhao2026opsd,he2026sdzero}. For multi-turn agents, this suggests a useful decomposition: RL supplies task-level optimization, while self-distillation supplies dense shaping signals~\citep{lu2026sdar}. The key question is where the privileged signal should come from. Existing methods often rely on generic revision contexts, external hints, or task-level feedback transformations. OPID instead constructs the privileged branch from hindsight skills extracted from on-policy trajectories, making the distillation signal directly tied to the states, actions, and failures encountered by the current policy.

\paragraph{Skill-conditioned agent learning.}
Natural-language skills provide compact procedural knowledge for agents, including subgoal decompositions, action templates, and failure-avoidance rules~\citep{lu2026skill0,wang2026skillsd,lu2026sdar,wu2026maestro}. Existing skill-based methods commonly depend on external skill libraries, retrieved skill files, or persistent skill memories. These designs can improve agent behavior, but they introduce maintenance and retrieval costs, and retrieved skills may be mismatched with the state distribution induced by the current policy. This mismatch becomes more severe in long-horizon interaction, where small deviations can lead to substantial state drift. OPID makes a different design choice: it extracts hierarchical skills directly from completed on-policy trajectories, routes them according to decision criticality, and distills their behavioral effect into the policy during training. As a result, OPID provides distribution-matched hindsight supervision without requiring skill retrieval, analyzer calls, or privileged context at inference time.

\section{Methods}

We formulate long-horizon agentic tasks as partially observable decision processes and present OPID, a framework that converts completed on-policy trajectories into hierarchical skills and distills their behavioral effect back into the policy. OPID{} performs on-policy skill distillation in three stages. First, it extracts hierarchical skills from completed on-policy trajectories. Second, it routes the appropriate skill to each decision step and converts the skill effect into token-level self-distillation signals. Third, it combines these token-level skill advantages with group-relative outcome advantages for policy optimization.
Figure~\ref{fig:overview} illustrates the overall pipeline.

\begin{figure}[t]
    \centering
    \includegraphics[width=\linewidth]{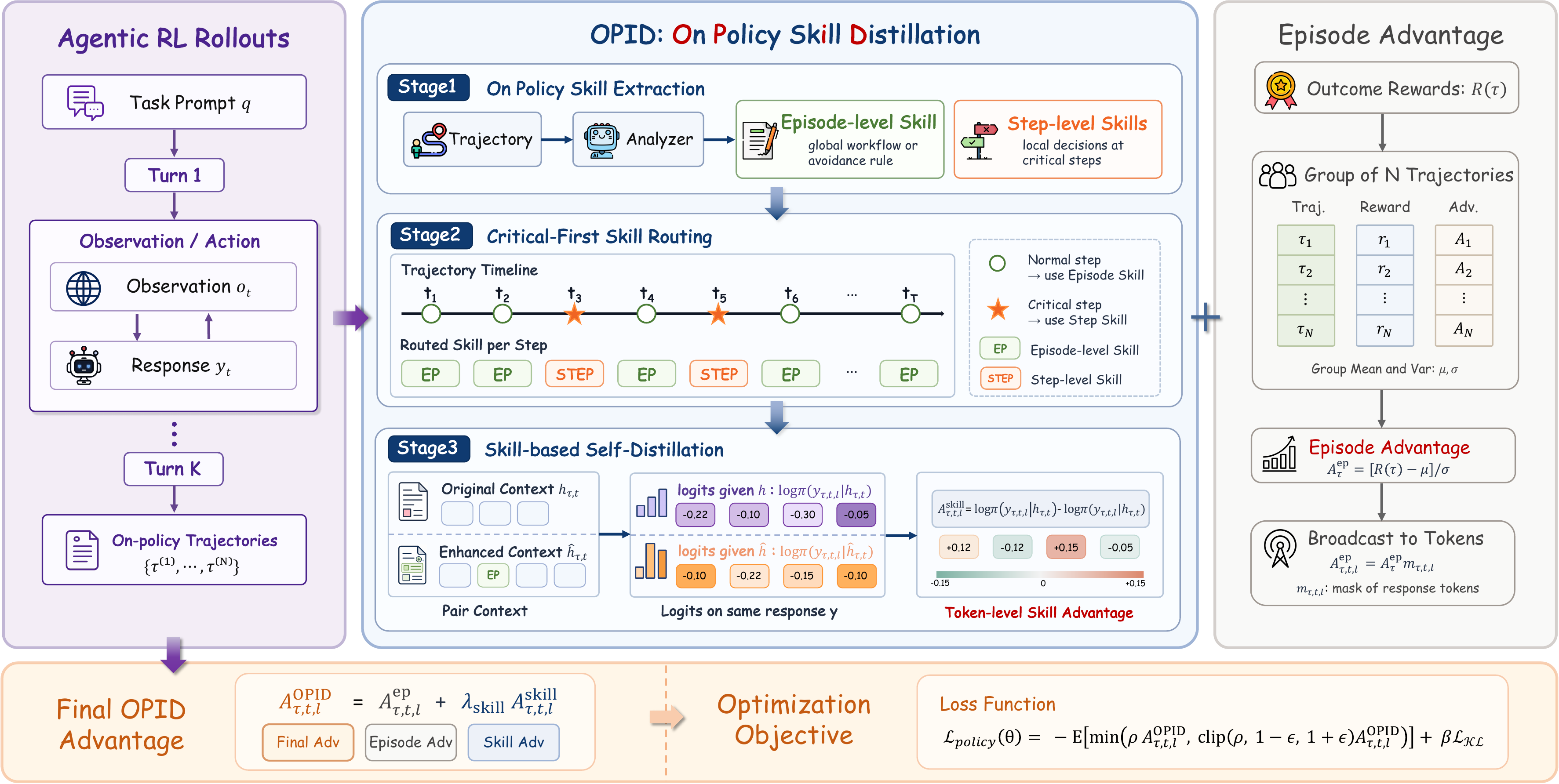}
    \caption{\textbf{Overview of OPID.} Starting from completed on-policy trajectories, OPID extracts hierarchical hindsight skills and routes the most relevant skill to each decision, prioritizing step-level skills at critical states. The policy then re-scores the same sampled response with and without the routed skill, turning the token-wise log-probability difference into a dense skill advantage that complements the episode-level RL signal.}
    \label{fig:overview}
\end{figure}

\subsection{Problem Formulation}

We model an agentic task as a partially observable Markov decision process defined by
\begin{equation*}
    (\mathcal{S}, \mathcal{A}, \mathcal{O}, \mathcal{T}, \mathcal{R}, \gamma),
\end{equation*}
where $\mathcal{S}$ is the latent state space, $\mathcal{A}$ is the action space, $\mathcal{O}$ is the observation space, $\mathcal{T}: \mathcal{S}\times\mathcal{A}\rightarrow\mathcal{S}$ is the transition function, $\mathcal{R}: \mathcal{S}\times\mathcal{A}\rightarrow \mathbb{R}$ is the reward function, and $\gamma \in [0,1)$ is the discount factor. At timestep $t$, the environment is in a hidden state $s_t\in\mathcal{S}$ and emits an observation $o_t\in\mathcal{O}$. The agent maintains an interaction history
\begin{equation*}
    h_t = (o_0, y_0, o_1, y_1, \ldots, o_t),
\end{equation*}
where $y_i$ denotes the textual response or executable action generated at step $i$. The policy $\pi_\theta$ generates the next response as
\begin{equation*}
    y_t \sim \pi_\theta(\cdot \mid h_t).
\end{equation*}

After executing $y_t$, the environment transitions and returns the next observation. A completed trajectory is represented as
\begin{equation*}
    \tau = \{(o_t, y_t, r_t)\}_{t=0}^{T-1},
\end{equation*}
where $T$ is the episode length. In most agentic benchmarks, rewards are sparse and terminal, so we denote the outcome score by
\begin{equation*}
    R(\tau) \in \{0,1\},
\end{equation*}
or more generally $R(\tau)\in\mathbb{R}$ when the benchmark provides graded feedback. The learning objective is
\begin{equation*}
    J(\pi_\theta) = \E_{\tau\sim\pi_\theta}[R(\tau)].
\end{equation*}

Following GRPO-style training, for each task prompt $q$ we sample a group of $N$ trajectories from the current policy:
\begin{equation*}
    \mathcal{G}_q = \{\tau^{(1)}, \tau^{(2)}, \ldots, \tau^{(N)}\}.
\end{equation*}

\subsection{On-Policy Skill Extraction}

Outcome rewards reveal whether a trajectory succeeds, but not why it succeeds or fails. OPID{} therefore represents post-hoc trajectory knowledge as hierarchical skills extracted from completed on-policy rollouts. The hierarchy contains two complementary levels.

\paragraph{Episode-level skills.}
An episode-level skill $s^{\mathrm{ep}}_{\tau}$ summarizes the global behavioral pattern of a complete trajectory $\tau$. For a successful trajectory, it captures a reusable workflow that explains how the task was solved. For a failed trajectory, it captures a failure-avoidance rule that describes what should be avoided in similar future situations. Episode-level skills are broad and stable, making them suitable as default guidance for most states.

\paragraph{Step-level skills.}
A step-level skill $s^{\mathrm{step}}_{\tau,t}$ captures local decision knowledge at timestep $t$. It is intended for pivotal states where the final outcome depends strongly on a specific choice, such as avoiding a repeated invalid action, selecting the next object to inspect, correcting a mistaken subgoal, or deciding when to stop exploration. Step-level skills are more precise than episode-level skills, but they are also sparse and state-dependent.

Given a completed trajectory $\tau$, OPID{} reconstructs an ordered trajectory record containing the task prompt, observations, model responses, environment feedback, step indices, and terminal outcome. An LLM-based analyzer $\mathcal{A}$ maps this record to structured natural-language skills:
\begin{equation*}
    \mathcal{A}(\tau) =
    \left(
        s^{\mathrm{ep}}_{\tau},
        \{s^{\mathrm{step}}_{\tau,t}\}_{t\in\mathcal{C}_\tau}
    \right),
\end{equation*}
where $\mathcal{C}_\tau$ is the sparse set of critical timesteps identified by the analyzer.

\subsection{Critical-First Skill-Conditioned Self-Distillation}

Applying the same skills to every step is suboptimal. Episode-level skills are robust but may be too coarse at decisive states, whereas step-level skills are precise but sparse. OPID{} therefore introduces critical-first skill routing before performing skill-conditioned self-distillation. For trajectory $\tau$ and timestep $t$, the routed skill is
\begin{equation*}
    s_{\tau,t} =
    \begin{cases}
    s^{\mathrm{step}}_{\tau,t}, & \text{if } t\in\mathcal{C}_\tau, \\
    s^{\mathrm{ep}}_{\tau}, & \text{otherwise}.
    \end{cases}
\end{equation*}
Equivalently, define routing masks
\newcommand{\ind}{\mathbb{I}}
\begin{equation*}
    q^{\mathrm{step}}_{\tau,t}=\ind[t\in\mathcal{C}_\tau],
    \qquad
    q^{\mathrm{ep}}_{\tau,t}=\ind[t\notin\mathcal{C}_\tau].
\end{equation*}
The critical-first rule enforces
\begin{equation*}
    q^{\mathrm{step}}_{\tau,t}=1 \Rightarrow q^{\mathrm{ep}}_{\tau,t}=0,
\end{equation*}
so the two skill levels are not blindly combined. Each step receives the most appropriate granularity.

After routing, OPID{} converts the selected skill into token-level self-distillation supervision. Let $H(\cdot,\cdot)$ denote a deterministic skill-injection function that appends or prepends the routed skill to the interaction history while preserving the original state information. The skill-augmented history is
\begin{equation*}
    \tilde{h}_{\tau,t} = H(h_{\tau,t}, s_{\tau,t}).
\end{equation*}

The original response $y_{\tau,t}$ is not regenerated. Instead, the old policy $\pi_{\theta_{\mathrm{old}}}$ scores the same sampled response under both the original and skill-augmented histories. For token $\ell$ in response $y_{\tau,t}$, define
\begin{equation*}
    \ell^{\mathrm{old}}_{\tau,t,\ell}
    =
    \log \pi_{\theta_{\mathrm{old}}}
    \left(
        y_{\tau,t,\ell}
        \mid
        h_{\tau,t}, y_{\tau,t,<\ell}
    \right),
\end{equation*}
and
\begin{equation*}
    \ell^{\mathrm{skill}}_{\tau,t,\ell}
    =
    \log \pi_{\theta_{\mathrm{old}}}
    \left(
        y_{\tau,t,\ell}
        \mid
        \tilde{h}_{\tau,t}, y_{\tau,t,<\ell}
    \right).
\end{equation*}
The skill-based self-teacher advantage is
\begin{equation*}
    A^{\mathrm{skill}}_{\tau,t,\ell}
    =
    \left(
        \ell^{\mathrm{skill}}_{\tau,t,\ell}
        -
        \ell^{\mathrm{old}}_{\tau,t,\ell}
    \right)
    m_{\tau,t,\ell},
\end{equation*}
where $m_{\tau,t,\ell}\in\{0,1\}$ is the valid response-token mask.

If $A^{\mathrm{skill}}_{\tau,t,\ell}>0$, the selected skill makes the token more likely under the old policy, suggesting that the token is consistent with the skill. If \(A^{\mathrm{skill}}_{\tau,t,\ell}<0\), the skill-conditioned context assigns lower probability to the token, suggesting that the token is less aligned with the routed hindsight skill. This procedure yields dense token-level guidance without requiring an external expert action.

\subsection{Policy Optimization with Skill Advantage}

For each rollout group $\mathcal{G}_q$, let
$\mathbf{r}_q=\{R(\tau')\mid \tau'\in\mathcal{G}_q\}$ denote the set of outcome rewards of all trajectories sampled for the same prompt. Following GRPO, the group mean is defined as
\begin{equation*}
    \mu_q
    =
    \operatorname{mean}(\mathbf{r}_q)
    =
    \frac{1}{|\mathcal{G}_q|}
    \sum_{\tau'\in\mathcal{G}_q} R(\tau').
\end{equation*}
The group standard deviation is defined as the square root of the group reward variance:
\begin{equation*}
    \sigma_q
    =
    \operatorname{std}(\mathbf{r}_q)
    =
    \sqrt{
    \frac{1}{|\mathcal{G}_q|}
    \sum_{\tau'\in\mathcal{G}_q}
    \left(R(\tau')-\mu_q\right)^2
    }.
\end{equation*}
The GRPO-style episode-relative advantage is then computed by normalizing the trajectory outcome reward within its prompt group:
\begin{equation*}
    A^{\mathrm{ep}}_{\tau}
    =
    \frac{R(\tau)-\mu_q}{\sigma_q},
    \qquad \tau\in\mathcal{G}_q.
\end{equation*}
This scalar is broadcast to all valid response tokens:
\begin{equation*}
    A^{\mathrm{ep}}_{\tau,t,\ell}
    =
    A^{\mathrm{ep}}_{\tau} m_{\tau,t,\ell}.
\end{equation*}

The final OPID advantage combines group-relative outcome feedback with token-level skill supervision:
\begin{equation*}
    A^{\text{OPID}}_{\tau,t,\ell}
    =
    A^{\mathrm{ep}}_{\tau,t,\ell}
    +
    \lambda_{\mathrm{skill}}
    A^{\mathrm{skill}}_{\tau,t,\ell}.
\end{equation*}

This formulation keeps outcome reward as the primary RL signal while adding token-level shaping.

We optimize the standard clipped policy objective:
\begin{equation*}
    \mathcal{L}_{\mathrm{policy}}(\theta)
    =
    -
    \mathbb{E}_{\tau,t,\ell}
    \left[
    \min
    \left(
        \rho_{\tau,t,\ell}(\theta) A^{\text{OPID}}_{\tau,t,\ell},
        \operatorname{clip}
        \left(
            \rho_{\tau,t,\ell}(\theta),
            1-\epsilon,
            1+\epsilon
        \right)
        A^{\text{OPID}}_{\tau,t,\ell}
    \right)
    \right] +\beta \mathcal L_{\mathrm{KL}}(\theta).
\end{equation*}

where \(\rho_{\tau,t,\ell}(\theta)\) denotes the token-level importance ratio, defined as
\begin{equation*}
    \rho_{\tau,t,\ell}(\theta)
    =
    \exp\left(
        \log\pi_{\theta}
        (y_{\tau,t,\ell}\mid h_{\tau,t},y_{\tau,t,<\ell})
        -
        \log\pi_{\theta_{\mathrm{old}}}
        (y_{\tau,t,\ell}\mid h_{\tau,t},y_{\tau,t,<\ell})
    \right).
\end{equation*}

The operator \(\operatorname{clip}(x,1-\epsilon,1+\epsilon)\) truncates \(x\) to the interval
\([1-\epsilon,1+\epsilon]\), and \(\epsilon\) is the clipping hyperparameter that controls the maximum allowed deviation from the old policy.

\paragraph{Training-inference boundary.}
The analyzer, routed skills, and skill-conditioned scoring pass are used only to construct the training advantage. At inference time, the learned policy acts from the ordinary interaction history $h_t$ alone, with no analyzer call, skill retrieval, or privileged context. 

\begin{table*}[ht!]
    \centering
    \caption{
        \textbf{Performance Comparison on the representative long-horizon benchmarks (ALFWorld, Search-based QA, and WebShop).}
        We report the success rate (\%) on ALFWorld, accuracy on search-based QA, and task-completion score/success rate on WebShop. An asterisk (*) denotes validation with skills. The \sethlcolor{topcolor}\hl{\textbf{best}} and \sethlcolor{secondcolor}\hl{\mbox{\underline{second-best}}} results are highlighted.
    }
    \vspace{0.05in}
    \label{tab:main_results}
    \resizebox{1\textwidth}{!}{%
    \begin{tabular}{l ccccccc cccccccc cc}
    \toprule
    & \multicolumn{7}{c}{\textbf{ALFWorld}} & \multicolumn{8}{c}{\textbf{Search-based QA}} & \multicolumn{2}{c}{\textbf{WebShop}} \\
    \cmidrule(lr){2-8} \cmidrule(lr){9-16} \cmidrule(lr){17-18}
    \textbf{Method}
    & \textbf{Pick} & \textbf{Look} & \textbf{Clean} & \textbf{Heat} & \textbf{Cool} & \textbf{Pick2} & \textbf{Avg}
    & \textbf{NQ} & \textbf{Triv} & \textbf{Pop} & \textbf{Hotp} & \textbf{2Wk} & \textbf{MuS} & \textbf{Bam} & \textbf{Avg}
    & \textbf{Score} & \textbf{Succ.} \\
    \midrule
    \rowcolor{gray!10} \multicolumn{18}{l}{\textit{Qwen2.5-3B-Instruct}} \\

    Vanilla
        & 44.4 & 11.1 & 6.2 & 15.4 & 28.6 & 12.5 & 21.9
        & 24.6 & 48.1 & 31.0 & 26.3 & 25.3 & 7.2 & 59.7 & 31.7
        & 6.7 & 0.8
        \\
    Skill-Prompt*
        & 51.7 & 66.7 & 48.4 & 0.0 & 4.3 & 10.0 & 28.9
        & 23.7 & 46.2 & 30.6 & 24.4 & 22.1 & 7.5 & 12.5 & 23.9
        & 0.2 & 0.8
        \\
    OPSD
        & 48.8 & 41.7 & 16.7 & 0.0 & 15.8 & 16.7 & 28.1
        & 0.1 & 0.1 & 0.1 & 0.0 & 0.0 & 0.0 & 0.0 & 0.0
        & 11.3 & 3.1
        \\
    GRPO
        & 91.2 & 62.5 & \cellcolor{secondcolor}\underline{96.2} & 61.9 & 65.0 & 47.4 & 75.0
        & 39.3 & 60.6 & 41.1 & 37.4 & 34.6 & 15.4 & 26.4 & 36.4
        & 79.8 & 63.3
        \\
    Skill-GRPO
        & 88.9 & 71.4 & 58.8 & \cellcolor{secondcolor}\underline{70.6} & 40.7 & 29.2 & 60.2
        & 43.5 & 58.8 & 43.0 & 36.8 & 32.2 & 11.7 & 12.5 & 34.1
        & 77.3 & 60.9
        \\
    Skill-GRPO*
        & 94.3 & 57.1 & \cellcolor{topcolor}\textbf{100.0} & 66.7 & 73.1 & 57.1 & 80.5
        & 44.3 & 59.6 & 44.3 & 39.0 & 36.1 & 14.5 & 14.9 & 36.1
        & 76.3 & 66.4
        \\
    GRPO+OPSD
        & \cellcolor{topcolor}\textbf{100.0} & \cellcolor{secondcolor}\underline{82.4} & 85.7 & \cellcolor{topcolor}\textbf{75.0} & 70.0 & 60.0 & 81.2
        & \cellcolor{secondcolor}\underline{44.9} & \cellcolor{secondcolor}\underline{61.2} & \cellcolor{secondcolor}\underline{45.2} & \cellcolor{secondcolor}\underline{40.4} & 38.5 & 16.0 & \cellcolor{secondcolor}\underline{66.1} & \cellcolor{secondcolor}\underline{44.6}
        & 77.8 & 66.4
        \\
    Skill-SD
        & 88.2 & 50.0 & \cellcolor{secondcolor}\underline{96.2} & 52.4 & 65.0 & 57.9 & 73.4
        & 44.4 & 60.4 & 44.0 & 39.5 & \cellcolor{topcolor}\textbf{40.4} & 15.4 & 64.9 & 44.1
        & 75.9 & 64.0
        \\
    RLSD
        & 87.9 & 75.0 & 90.9 & \cellcolor{topcolor}\textbf{75.0} & 73.1 & 68.4 & 79.7
        & 41.5 & 58.6 & 42.3 & \cellcolor{secondcolor}\underline{40.4} & \cellcolor{secondcolor}\underline{40.2} & \cellcolor{topcolor}\textbf{16.8} & \cellcolor{topcolor}\textbf{66.9} & 43.8
        & \cellcolor{secondcolor}\underline{84.4} & 66.4
        \\
    SDAR
        & \cellcolor{secondcolor}\underline{97.1} & 62.5 & \cellcolor{topcolor}\textbf{100.0} & 61.9 & \cellcolor{secondcolor}\underline{75.0} & \cellcolor{topcolor}\textbf{84.2} & \cellcolor{topcolor}\textbf{84.4}
        & 44.8 & 58.1 & 44.3 & 38.6 & 36.2 & 15.7 & \cellcolor{secondcolor}\underline{66.1} & 43.4
        & \cellcolor{topcolor}\textbf{85.0} & \cellcolor{secondcolor}\underline{68.0}
        \\
    OPID
        & 92.7 & \cellcolor{topcolor}\textbf{100.0} & 88.9 & 70.0 & \cellcolor{topcolor}\textbf{84.2} & \cellcolor{secondcolor}\underline{70.0} & \cellcolor{secondcolor}\underline{84.3}
        & \cellcolor{topcolor}\textbf{45.9} & \cellcolor{topcolor}\textbf{61.4} & \cellcolor{topcolor}\textbf{45.7} & \cellcolor{topcolor}\textbf{40.7} & 38.8 & \cellcolor{secondcolor}\underline{16.4} & \cellcolor{secondcolor}\underline{66.1} & \cellcolor{topcolor}\textbf{45.0}
        & \cellcolor{topcolor}\textbf{85.0} & \cellcolor{topcolor}\textbf{74.2}
        \\

    \midrule

    \rowcolor{gray!10} \multicolumn{18}{l}{\textit{Qwen2.5-7B-Instruct}} \\

    Vanilla
        & 36.1 & 22.2 & 3.1 & 0.0 & 0.0 & 0.0 & 12.5
        & 25.2 & 50.8 & 29.5 & 29.0 & 29.0 & 10.4 & 63.7 & 33.9
        & 5.9 & 1.6
        \\
    Skill-Prompt*
        & 51.7 & 50.0 & 32.3 & 5.3 & 4.3 & 0.0 & 23.4
        & 30.9 & 52.1 & 32.7 & 32.7 & 27.9 & 12.7 & 66.1 & 36.4
        & 1.7 & 0.8
        \\
    OPSD
        & 50.0 & 60.0 & 22.7 & 21.4 & 17.6 & 9.5 & 32.8
        & 8.8 & 8.6 & 17.5 & 2.5 & 4.2 & 0.5 & 1.2 & 6.2
        & 4.5 & 2.3
        \\
    GRPO
        & 91.2 & \cellcolor{secondcolor}\underline{87.5} & 96.2 & 81.0 & 65.0 & 57.9 & 81.2
        & 45.1 & 63.7 & 44.0 & 43.6 & 43.2 & 16.8 & 37.6 & 42.0
        & 80.9 & 72.6
        \\
    Skill-GRPO
        & 88.5 & 66.7 & 65.2 & 61.1 & 57.7 & 73.1 & 69.5
        & 45.2 & 63.7 & 45.7 & 43.1 & 43.3 & 19.6 & 21.4 & 40.3
        & 80.4 & 71.9
        \\
    Skill-GRPO*
        & \cellcolor{topcolor}\textbf{100.0} & 83.3 & 96.4 & 83.3 & 75.0 & \cellcolor{secondcolor}\underline{78.9} & \cellcolor{secondcolor}\underline{88.3}
        & 44.8 & 63.0 & 45.1 & 43.7 & 43.7 & 20.5 & 71.4 & 47.5
        & 87.0 & \cellcolor{secondcolor}\underline{81.2}
        \\
    GRPO+OPSD
        & 91.4 & 61.5 & \cellcolor{topcolor}\textbf{100.0} & \cellcolor{secondcolor}\underline{87.5} & 76.5 & 52.2 & 80.4
        & \cellcolor{secondcolor}\underline{47.3} & \cellcolor{secondcolor}\underline{64.5} & 46.9 & 43.8 & 39.3 & 18.0 & 69.4 & 47.0
        & 86.8 & 76.5
        \\
    Skill-SD
        & 93.9 & \cellcolor{topcolor}\textbf{93.8} & 90.9 & \cellcolor{topcolor}\textbf{100.0} & 69.2 & 68.4 & 85.1
        & 47.1 & \cellcolor{secondcolor}\underline{64.5} & \cellcolor{secondcolor}\underline{47.8} & 44.2 & 42.1 & 20.2 & 69.0 & 47.8
        & 86.1 & 76.5
        \\
    RLSD
        & \cellcolor{topcolor}\textbf{100.0} & \cellcolor{secondcolor}\underline{87.5} & 92.3 & 58.8 & \cellcolor{secondcolor}\underline{80.0} & 65.2 & 82.0
        & 46.8 & 63.0 & 44.4 & \cellcolor{secondcolor}\underline{45.5} & \cellcolor{topcolor}\textbf{48.9} & \cellcolor{secondcolor}\underline{21.5} & \cellcolor{topcolor}\textbf{73.0} & \cellcolor{secondcolor}\underline{49.0}
        & \cellcolor{secondcolor}\underline{87.4} & 77.3
        \\
    SDAR
        & \cellcolor{secondcolor}\underline{94.7} & 75.0 & \cellcolor{topcolor}\textbf{100.0} & 86.7 & 68.2 & \cellcolor{secondcolor}\underline{78.9} & 85.9
        & 46.3 & 63.5 & \cellcolor{topcolor}\textbf{48.2} & 43.8 & \cellcolor{secondcolor}\underline{48.4} & 19.6 & \cellcolor{topcolor}\textbf{73.0} & \cellcolor{secondcolor}\underline{49.0}
        & \cellcolor{topcolor}\textbf{89.4} & \cellcolor{topcolor}\textbf{82.8}
        \\
    OPID
        & \cellcolor{topcolor}\textbf{100.0} & 81.8 & \cellcolor{secondcolor}\underline{97.1} & \cellcolor{topcolor}\textbf{100.0} & \cellcolor{topcolor}\textbf{80.8} & \cellcolor{topcolor}\textbf{80.0} & \cellcolor{topcolor}\textbf{90.0}
        & \cellcolor{topcolor}\textbf{48.8} & \cellcolor{topcolor}\textbf{65.6} & 46.8 & \cellcolor{topcolor}\textbf{46.1} & 42.7 & \cellcolor{topcolor}\textbf{21.7} & \cellcolor{secondcolor}\underline{72.6} & \cellcolor{topcolor}\textbf{49.2}
        & 85.3 & 79.7
        \\

    \midrule

    \rowcolor{gray!10} \multicolumn{18}{l}{\textit{Qwen3-1.7B-Instruct}} \\

    Vanilla
        & 25.0 & 22.2 & 3.1 & 0.0 & 21.4 & 4.2 & 12.5
        & 29.4 & 46.9 & 37.0 & 23.5 & 19.6 & 6.4 & 10.5 & 24.8
        & 46.5 & 4.7
        \\
    Skill-Prompt*
        & 10.3 & 50.0 & 16.1 & 0.0 & 0.0 & 5.0 & 9.4
        & 29.4 & 46.5 & 36.2 & 22.9 & 20.8 & 4.3 & 10.1 & 24.3
        & 23.0 & 2.3
        \\
    OPSD
        & 26.3 & 33.3 & 9.1 & 0.0 & 4.5 & 5.3 & 14.1
        & 4.2 & 8.3 & 4.6 & 6.6 & 15.3 & 0.7 & 1.2 & 5.8
        & 47.4 & 9.3
        \\
    GRPO
        & \cellcolor{secondcolor}\underline{71.1} & 41.7 & 36.4 & \cellcolor{secondcolor}\underline{40.0} & 31.8 & 31.6 & 46.1
        & \cellcolor{secondcolor}\underline{40.0} & \cellcolor{topcolor}\textbf{58.9} & 43.5 & 35.4 & 30.3 & 12.0 & 65.7 & 40.8
        & 67.3 & 38.3
        \\
    Skill-GRPO
        & 27.6 & \cellcolor{secondcolor}\underline{54.5} & 22.7 & 27.3 & 0.0 & 19.2 & 21.1
        & 39.2 & \cellcolor{secondcolor}\underline{58.6} & 43.9 & 35.2 & 28.2 & 11.5 & \cellcolor{secondcolor}\underline{66.1} & 40.4
        & 73.4 & 46.1
        \\
    Skill-GRPO*
        & 31.4 & 42.9 & 51.9 & 8.3 & 11.5 & 7.1 & 28.1
        & 38.0 & 58.4 & 43.9 & \cellcolor{secondcolor}\underline{36.3} & 29.0 & 12.5 & \cellcolor{topcolor}\textbf{66.9} & 40.7
        & \cellcolor{secondcolor}\underline{80.4} & 50.0
        \\
    GRPO+OPSD
        & 38.2 & 50.0 & 30.8 & 28.6 & 30.0 & 21.1 & 32.0
        & \cellcolor{topcolor}\textbf{40.7} & \cellcolor{topcolor}\textbf{58.9} & 45.0 & \cellcolor{topcolor}\textbf{37.0} & \cellcolor{secondcolor}\underline{34.6} & \cellcolor{topcolor}\textbf{13.3} & 65.7 & \cellcolor{topcolor}\textbf{42.2}
        & 70.7 & 38.3
        \\
    Skill-SD
        & 52.9 & 37.5 & \cellcolor{secondcolor}\underline{69.2} & \cellcolor{topcolor}\textbf{42.9} & \cellcolor{secondcolor}\underline{60.0} & \cellcolor{secondcolor}\underline{36.8} & 52.3
        & 39.1 & 57.5 & \cellcolor{topcolor}\textbf{45.4} & 34.8 & 34.1 & 10.7 & 64.1 & 40.8
        & \cellcolor{topcolor}\textbf{81.8} & 53.9
        \\
    RLSD
        & 50.0 & 37.5 & 61.5 & 19.0 & 50.0 & 21.1 & 42.2
        & 38.6 & 57.3 & 43.0 & 34.5 & 34.1 & 11.5 & 65.3 & 40.6
        & 74.0 & 50.8
        \\
    SDAR
        & \cellcolor{topcolor}\textbf{73.5} & 25.0 & \cellcolor{topcolor}\textbf{76.9} & 33.3 & 40.0 & \cellcolor{secondcolor}\underline{36.8} & \cellcolor{secondcolor}\underline{53.9}
        & 39.7 & \cellcolor{topcolor}\textbf{58.9} & \cellcolor{secondcolor}\underline{45.3} & 35.9 & \cellcolor{topcolor}\textbf{35.5} & \cellcolor{secondcolor}\underline{12.6} & 65.3 & \cellcolor{secondcolor}\underline{41.9}
        & 76.8 & \cellcolor{secondcolor}\underline{58.6}
        \\
    OPID
        & 65.9 & \cellcolor{topcolor}\textbf{72.7} & 66.7 & \cellcolor{secondcolor}\underline{40.0} & \cellcolor{topcolor}\textbf{63.2} & \cellcolor{topcolor}\textbf{45.0} & \cellcolor{topcolor}\textbf{58.9}
        & 38.1 & 58.1 & 43.4 & 35.5 & 31.7 & 11.7 & 64.5 & 40.4
        & 79.6 & \cellcolor{topcolor}\textbf{64.8}
        \\

    \bottomrule
    \end{tabular}
    }
\end{table*}

\section{Experiment}
\subsection{Experimental Setting}

\paragraph{Benchmarks.}
We evaluate OPID on three representative agentic benchmarks that require multi-step interaction or search-based reasoning. First, we use ALFWorld~\citep{shridhar2020alfworld}, an embodied household benchmark where an agent must complete language-specified goals through a sequence of textual actions. We report performance on six task types: \textit{Pick}, \textit{Look}, \textit{Clean}, \textit{Heat}, \textit{Cool}, and \textit{Pick2}. Second, we evaluate on WebShop~\citep{yao2022webshop}, where an agent interacts with an e-commerce website to find and purchase products satisfying natural-language user requirements. Following the standard evaluation protocol, we report results on 128 test tasks. Third, we consider Search-based QA~\citep{jin2025searchr1}, where the agent answers questions by interacting with a search environment: Natural Questions~\citep{kwiatkowski2019natural}, TriviaQA~\citep{joshi2017triviaqa}, PopQA~\citep{mallen2023popqa}, HotpotQA~\citep{yang2018hotpotqa}, 2WikiMultiHopQA~\citep{ho2020constructing}, MuSiQue~\citep{trivedi2022musique}, and Bamboogle~\citep{press2023measuring}.

\paragraph{Baselines.}
We compare OPID against both prompting-based and training-based baselines. \textit{Vanilla} denotes the original prompting baseline. \textit{Skill-Prompt} augments the model with skill descriptions at inference or validation time. \textit{GRPO} is the outcome-only on-policy RL baseline, where the policy is optimized using group-relative trajectory-level rewards~\citep{shao2024deepseekmath}. \textit{Skill-GRPO} combines skill conditioning with GRPO-style outcome optimization. \textit{OPSD}~\citep{zhao2026opsd}, \textit{GRPO+OPSD}, \textit{Skill-SD}~\citep{wang2026skillsd}, \textit{RLSD}~\citep{yang2026rlsd}, and \textit{SDAR}~\citep{lu2026sdar} are self-distillation or skill-distillation baselines that introduce auxiliary token-level or skill-conditioned supervision during training. Rows marked with $*$ indicate validation with skills, following the setting described in the corresponding baseline.

\begin{figure}[t]
    \centering
    \begin{minipage}[t]{0.495\linewidth}
        \centering
        \includegraphics[width=\linewidth,trim={0 0 8pt 0},clip]{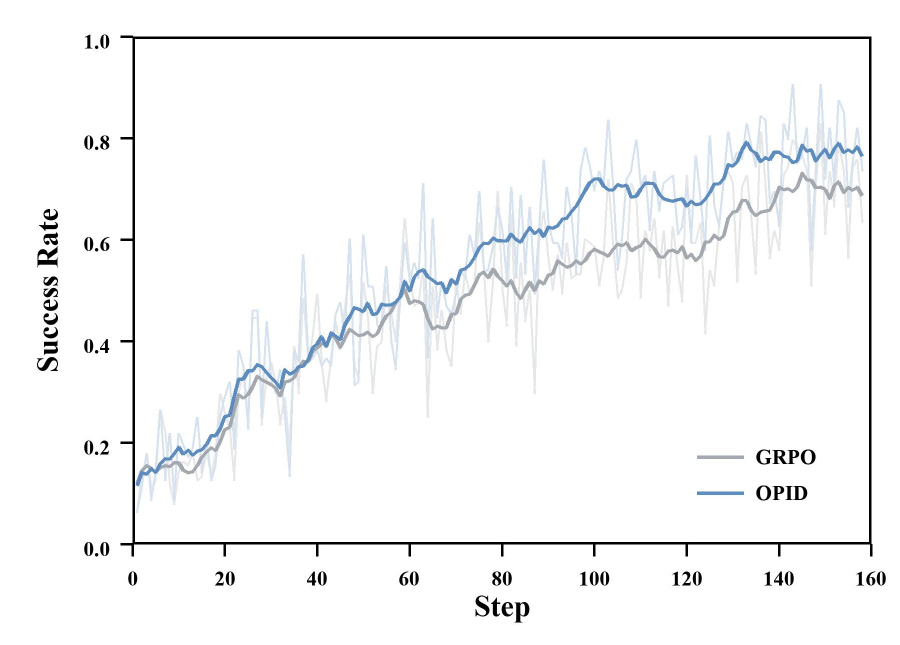}
        \centerline{(a) Episode success rate}
    \end{minipage}
    \hspace{-2mm}
    \begin{minipage}[t]{0.495\linewidth}
        \centering
        \includegraphics[width=\linewidth,trim={0 0 0 0},clip]{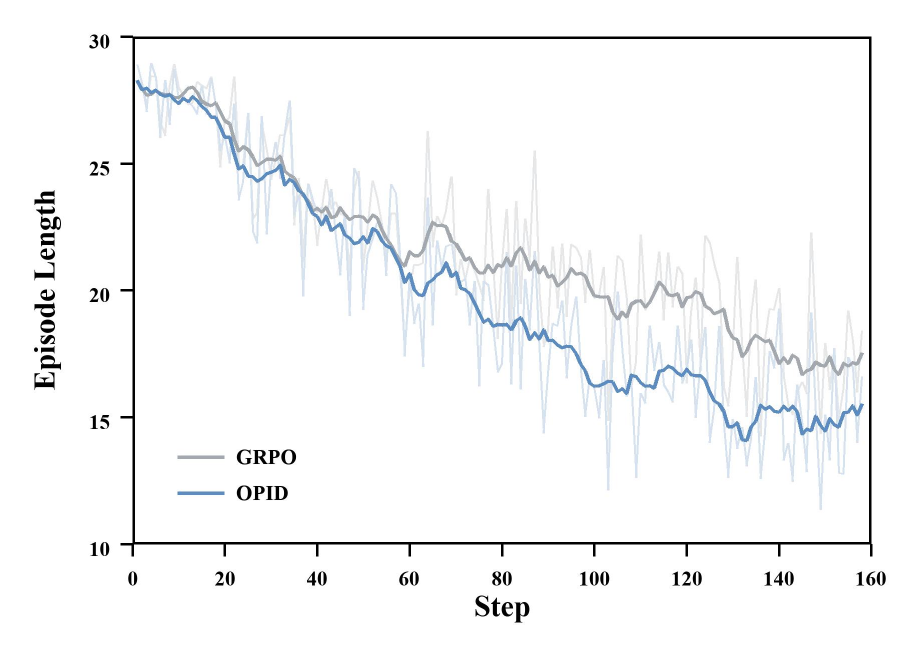}
        \centerline{(b) Episode length}
    \end{minipage}
    \caption{
        \textbf{Training dynamics of OPID and GRPO.}
        We report Qwen2.5-3B-Instruct training on ALFWorld. Translucent curves denote raw measurements and solid curves denote smoothed trends.
    }
    \label{fig:training_dynamics}
\end{figure}

\paragraph{Evaluation Metrics.}
For ALFWorld, we report task success rate in percentage. For WebShop, we report both the normalized task score and task success rate, following the benchmark protocol. For Search-based QA, we report answer accuracy in percentage on each QA subset and the average accuracy across subsets. 

\paragraph{Implementation Details.}
We conduct experiments using Qwen2.5-3B/7B-Instruct~\citep{yang2024qwen25} and Qwen3-1.7B-Instruct~\citep{yang2025qwen3}. The training batch size is set to 16 for ALFWorld and WebShop, and 128 for Search-based QA. All models are trained for 150 steps across all environments. Full details are provided in Appendix~\ref{app:experimental_details}.

\subsection{Main Results}\label{subsec32_main_results}
Table~\ref{tab:main_results} summarizes performance across model scales and agentic domains, revealing three key findings:

\paragraph{OPID consistently strengthens outcome-only RL.}
OPID improves over GRPO in most model--domain combinations. On Qwen2.5-3B, the gains are +9.3 points on ALFWorld (84.3 vs. 75.0), +8.6 on Search-based QA (45.0 vs. 36.4), and +10.9 on WebShop (74.2 vs. 63.3). The corresponding improvements on Qwen2.5-7B are +8.8, +7.2, and +7.1 points. The benefit is particularly pronounced for the smaller Qwen3-1.7B backbone, where OPID improves ALFWorld by +12.8 points and WebShop by +26.5 points. The only exception is Search-based QA on Qwen3-1.7B, where OPID remains close to GRPO. Overall, these results show that OPID usually provides a consistent gain over outcome-only reinforcement learning, especially on long-horizon embodied and web-shopping tasks.

\paragraph{OPID remains competitive with strong hybrid methods.}
Beyond improving over outcome-only RL, OPID also matches or surpasses strong hybrid and self-distillation baselines in several aggregate settings. On ALFWorld, OPID achieves the best average on Qwen2.5-7B and Qwen3-1.7B, outperforming the strongest baseline by +1.7 points (90.0 vs. 88.3) and +5.0 points (58.9 vs. 53.9) respectively. On Search-based QA, OPID attains the best average on both Qwen2.5 backbones, improving over the strongest baseline by +0.4 points on Qwen2.5-3B (45.0 vs. 44.6) and +0.2 points on Qwen2.5-7B (49.2 vs. 49.0). On WebShop, OPID achieves the best success rate on Qwen2.5-3B and Qwen3-1.7B, exceeding the strongest competing method by +6.2 points on Qwen3-1.7B (64.8 vs. 58.6), while remaining competitive on Qwen2.5-7B. These results show that trajectory-derived, distribution-matched skills can complement outcome supervision and compete with methods that rely on hybrid training signals or external skill contexts.

\begin{figure}[htbp]
\centering
\begin{minipage}[t]{0.48\linewidth}
    \centering
    \includegraphics[width=\linewidth]{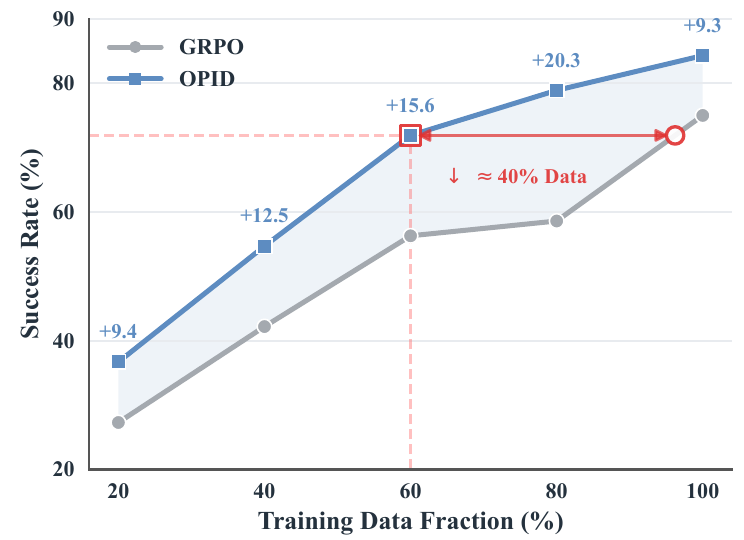}
    \vspace{-0.15in}
    \caption{
        \textbf{Sample efficiency analysis.}
        OPID consistently outperforms GRPO under reduced training data and approaches full-data GRPO performance using about 60\% of the data.
    }
    \label{fig:sample_efficiency_line}
\end{minipage}
\hfill
\begin{minipage}[t]{0.48\linewidth}
    \centering
    \includegraphics[width=\linewidth]{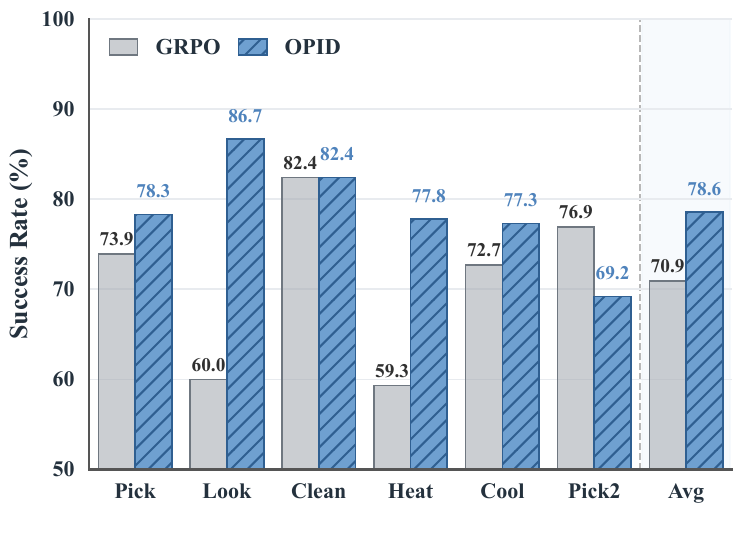}
    \vspace{-0.15in}
    \caption{
        \textbf{Cross-domain generalization on ALFWorld Unseen.}
        OPID improves the average success rate over GRPO and shows particularly large gains on \textit{Look} and \textit{Heat}.
    }
    \label{fig:generalization_bar}
\end{minipage}
\end{figure}

\paragraph{OPID internalizes skills instead of depending on them at inference.}
The results further show that OPID gains from internalizing hindsight skills into the policy, rather than relying on skill prompts at inference time. Training directly with retrieved skills introduces a clear train--test context mismatch: when validation-time skills are removed, Skill-GRPO underperforms ordinary GRPO on ALFWorld at all model scales, dropping by -14.8 points on Qwen2.5-3B (60.2 vs. 75.0), -11.7 points on Qwen2.5-7B (69.5 vs. 81.2), and -25.0 points on Qwen3-1.7B (21.1 vs. 46.1). In contrast, OPID is also evaluated without any skill input, yet exceeds Skill-GRPO by +24.1, +20.5, and +37.8 points. On Search-based QA, OPID also improves over both GRPO and Skill-GRPO for the two Qwen2.5 models, with gains over GRPO of +8.6 and +7.2 points, while remaining comparable on Qwen3-1.7B. Moreover, OPID outperforms Skill-GRPO* on ALFWorld and Search-based QA for both Qwen2.5 backbones, even though Skill-GRPO* retains privileged skill context during validation. These results indicate that OPID transfers trajectory-derived hindsight knowledge into the model parameters, enabling the policy to benefit from skills without depending on external skill prompts at inference.

\subsection{Training Dynamics}
Figure~\ref{fig:training_dynamics} illustrates the training progression on ALFWorld. Both methods improve during early optimization, yet OPID diverges from GRPO in the middle stage and maintains superior performance throughout the remainder of training. This divergence pattern indicates that hindsight skill supervision accelerates policy refinement beyond what outcome rewards alone can achieve. The efficiency gains are equally pronounced. OPID reduces average episode length to 15-16 steps while GRPO plateaus at 17-18 steps. The concurrent rise in success and fall in trajectory length reveals a key behavioral shift: OPID agents learn to reach goals through more direct action sequences rather than exploratory detours.

These dynamics align with the intended function of hierarchical supervision. Episode-level skills establish coherent task workflows that reduce backtracking and repetition. Step-level skills provide precise guidance at critical decision points, preventing the invalid actions and local navigation errors that otherwise extend trajectories. Together, these mechanisms enable OPID to internalize both global task structure and local decision efficiency.

\subsection{Sample Efficiency}\label{subsec33_efficiency_analysis}
Figure~\ref{fig:sample_efficiency_line} compares OPID and GRPO under different fractions of ALFWorld training data. OPID consistently improves over GRPO across all data scales, with absolute gains ranging from +9.3 to +20.3 points. The advantage is especially clear in the low- and mid-data regimes, where each trajectory carries more training value. With 60\% of the data, OPID reaches 71.9, close to GRPO trained with the full dataset (75.0); with 80\% of the data, it already surpasses full-data GRPO (78.9 vs. 75.0). These results indicate that OPID-style skill supervision improves the data efficiency of outcome-based RL. By converting completed trajectories into dense token-level training signals, OPID extracts additional supervision from the same environment interactions rather than relying only on terminal rewards. This makes the optimization less dependent on large numbers of rollouts and allows the policy to acquire effective behaviors with fewer samples.

\subsection{Cross-Domain Generalization}\label{subsec34_cross_domain_generalization}
Figure~\ref{fig:generalization_bar} evaluates cross-domain transfer to the ALFWorld unseen split. OPID achieves an average success rate of 78.6, outperforming GRPO by +7.7 points. Its gains over GRPO are concentrated on tasks like \textit{Look} (+26.7) and \textit{Heat} (+18.5), while maintaining competitive performance on the remaining task types. These results suggest that OPID is not merely memorizing the observed training trajectories. Instead, the extracted skills appear to capture reusable behavioral structure, including high-level task workflows and local decision rules that remain useful under unseen environment configurations. Since the skills are distilled into the policy rather than retrieved at inference time, the improvement also indicates that OPID internalizes transferable decision knowledge into the model parameters.

\begin{table}[ht!]
\centering
\caption{
    \textbf{Ablation on Hierarchical Skills.}
    We report the success rate (\%) on ALFWorld and Score/Succ. (\%) on WebShop with Qwen2.5-3B-Instruct backbone.
}
\vspace{0.05in}
\label{tab:ablation_hierarchical_skills}
\small
\setlength{\tabcolsep}{5.5pt}
\renewcommand{\arraystretch}{1.08}
\begin{tabular}{lccccccccc}
\toprule
& \multicolumn{7}{c}{\textbf{ALFWorld}}
& \multicolumn{2}{c}{\textbf{WebShop}} \\
\cmidrule(lr){2-8} \cmidrule(lr){9-10}
\textbf{Method}
& \textbf{Pick} & \textbf{Look} & \textbf{Clean} & \textbf{Heat} & \textbf{Cool} & \textbf{Pick2} & \textbf{Avg.}
& \textbf{Score} & \textbf{Succ.} \\
\midrule
\rowcolor{gray!10}
OPID
    & 92.7 & \textbf{100.0} & \textbf{88.9} & \textbf{70.0} & \textbf{84.2} & 70.0 & \textbf{84.3}
    & \textbf{85.0} & \textbf{74.2}
    \\
\midrule
\;\;\;w/o episode skill
    & 83.3 & 80.0 & 78.1 & 69.2 & 57.7 & \textbf{76.5} & 74.1
    & 78.4 & 67.2
    \\
\;\;\;w/o step skill
    & \textbf{95.1} & \textbf{81.8} & \textbf{88.9} & \textbf{70.0} & 79.0 & 60.0 & 79.1
    & 80.2 & 65.6
    \\
\bottomrule
\end{tabular}
\end{table}
\begin{table}[ht!]
\centering
\caption{
    \textbf{Ablation of Critical-First Skill Routing.}
    With the Qwen2.5-3B-Instruct backbone, we compare OPID with a variant that removes the critical-first routing strategy.
}
\vspace{0.05in}
\label{tab:ablation_critical_first}
\small
\renewcommand{\arraystretch}{1.08}
\begin{tabular}{lccccccc}
\toprule
& \multicolumn{7}{c}{\textbf{ALFWorld}} \\
\cmidrule(lr){2-8}
\textbf{Method}
& \textbf{Pick} & \textbf{Look} & \textbf{Clean} & \textbf{Heat} & \textbf{Cool} & \textbf{Pick2} & \textbf{Avg.} \\
\midrule
\rowcolor{gray!10}
OPID
    & 92.7 & \textbf{100.0} & \textbf{88.9} & \textbf{70.0} & \textbf{84.2} & \textbf{70.0} & \textbf{84.3}
    \\
\;\;\;w/o Routing
    & \textbf{95.1} & 81.8 & \textbf{88.9} & 50.0 & \textbf{84.2} & 65.0 & 77.5
    \\
\bottomrule
\end{tabular}
\end{table}

\subsection{Ablation Studies and Analysis}
We isolate the contributions of hierarchical skill granularity and critical-first routing using Qwen2.5-3B-Instruct.

\paragraph{Impact of Hierarchical Skills.}
As shown in Table~\ref{tab:ablation_hierarchical_skills}, the complete hierarchy obtains the best aggregate performance on both domains. Removing episode-level skills decreases the ALFWorld average from 84.3 to 74.1 and the WebShop success rate from 74.2 to 67.2, confirming that global workflows and failure-avoidance rules provide an important default signal. Removing step-level skills decreases the ALFWorld average from 84.3 to 79.1 and the WebShop success rate from 74.2 to 65.6. These results demonstrate the complementarity of the two skill levels.

\paragraph{Impact of Critical-First Skill Routing.}
Table~\ref{tab:ablation_critical_first} compares OPID with a non-routed variant that applies the episode-level skill to every step and additionally incorporates the corresponding step-level skill at critical timesteps, thereby superimposing the two forms of guidance. Critical-first routing improves the ALFWorld average by +6.8 points (84.3 vs. 77.5). These results show that selectively routing the most appropriate skill granularity is more effective than directly combining global and local guidance, demonstrating the importance of critical-first routing.

\begin{figure}[ht!]
    \centering
    \includegraphics[width=\linewidth]{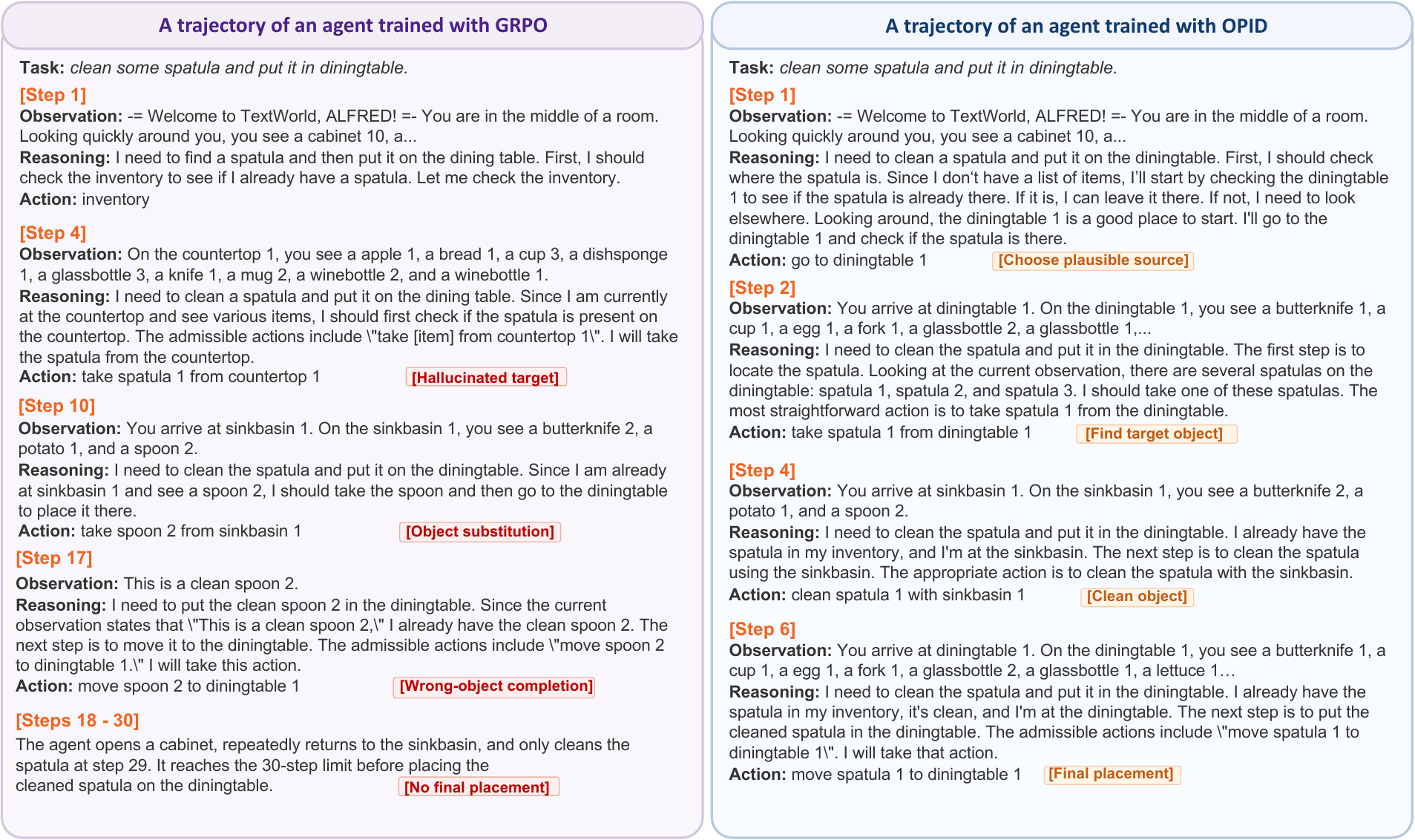}
    \caption{
        \textbf{Qualitative comparison on ALFWorld.}
        For the task ``clean some spatula and put it in diningtable,'' the GRPO-trained agent hallucinates a nonexistent target object, substitutes a spoon for the spatula, and fails to complete the final placement within the step limit. In contrast, OPID follows a coherent locate-clean-place workflow, grounding each action in the current observation and completing the task in six steps.
    }
    \label{fig:case_comparison}
\end{figure}

\paragraph{Qualitative Analysis. }
Figure~\ref{fig:case_comparison} illustrates an ALFWorld clean-and-place task. The GRPO-trained agent exhibits a “hallucinated target” error by attempting to take a nonexistent spatula from the countertop at Step 4. It subsequently substitutes a spoon for the target object and reaches the 30-step limit before placing the cleaned spatula back on the dining table. In contrast, OPID follows a coherent locate–clean–place workflow and completes the task in six steps. This case suggests that distilling hierarchical hindsight skills from on-policy trajectories helps the agent learn both local object-grounding decisions and episode-level task workflows, thereby reducing hallucinated actions and preserving progress toward the final goal.

\section{Conclusion}
We presented OPID, an on-policy skill distillation framework that turns completed agent trajectories into hierarchical hindsight supervision. By extracting episode-level and step-level skills from the current policy's own rollouts, OPID provides dense, distribution-matched token-level guidance while preserving outcome-based RL as the primary objective. Experiments across embodied, web, and search-based agentic benchmarks show that OPID improves agent learning without relying on external skill libraries, retrieval, or privileged context at inference time. More broadly, our results suggest that agent trajectories are not only samples for reward optimization, but also reusable records of decision knowledge that can be distilled back into the policy.

\bibliography{iclr2025_conference}
\bibliographystyle{iclr2025_conference}

\appendix

\section{Theoretical Analysis}
\label{app:theory}

This section provides three results that correspond to the main design choices of
OPID. We first place the proposed teacher advantage among representative
on-policy distillation objectives. We then show that it implements a sampled-token
reverse-KL update, characterize the benefit of collecting distillation contexts
on policy, and justify critical-first routing under a natural specialization
assumption.
\subsection{Notation and Representative On-Policy Distillation Objectives}
\label{app:theory_setup}

\subsubsection{Notation}

Let \(i=(\tau,t,\ell)\) index a valid token position in a response. We denote
the corresponding standard autoregressive context by
\(c_i=(h_{\tau,t},y_{\tau,t,<\ell})\), and its skill-augmented counterpart by
\(\widetilde c_i\). At each token position, define
\begin{equation*}
\begin{aligned}
 b_i(v) &\triangleq \pi_{\theta_{\mathrm{old}}}(v\mid c_i),
 & q_i(v) &\triangleq \pi_{\theta_{\mathrm{old}}}(v\mid \widetilde c_i),
 & p_{\theta,i}(v) &\triangleq \pi_{\theta}(v\mid c_i).
\end{aligned}
\end{equation*}
Here, \(b_i\) is the behavior distribution used to generate the response,
\(q_i\) is a detached skill-conditioned teacher distribution, and
\(p_{\theta,i}\) is the trainable policy evaluated under the standard context
available at inference time. The observed token
\(a_i\triangleq y_{\tau,t,\ell}\) is sampled from \(b_i\).

We further define the token-level log-likelihood gap and the policy importance
ratio as
\begin{equation}
\begin{aligned}
 \Delta_i(v) &\triangleq \log q_i(v)-\log b_i(v),
 & \rho_{\theta,i}(v) &\triangleq
 \frac{p_{\theta,i}(v)}{b_i(v)}.
\end{aligned}
\label{eq:gap_ratio}
\end{equation}
The quantity \(\Delta_i(v)\) measures the change in token log-probability
induced by the skill-augmented context. In particular,
\(\Delta_i(v)>0\) indicates that the skill-conditioned teacher assigns greater
probability to token \(v\) than the behavior policy does. The OPID skill
advantage associated with the observed token is therefore
\[
A_i^{\mathrm{skill}}=\Delta_i(a_i).
\]
Unless otherwise stated, all expectations below are taken over valid response
tokens; the response mask is consequently omitted for notational simplicity.

\subsubsection{Representative On-Policy Distillation Objectives}

On-policy distillation (OPD) applies teacher supervision at autoregressive
contexts generated by the student or a behavior policy, thereby reducing the
context-distribution mismatch between distillation training and free-running
inference \citep{agarwal2024onpolicy}. The context-generation policy and the
granularity of teacher supervision are orthogonal design choices. At each
on-policy context, output-space OPD objectives can be organized into three
common supervision granularities: full-vocabulary, Top-$K$, and sampled-token
distillation \citep{li2026rethinkingopd,fu2026revisitingopd}. OPID belongs to
the sampled-token category.

\paragraph{Full-vocabulary distribution matching.}

Let \(q_i\) and \(p_{\theta,i}\) denote the teacher and student next-token
distributions, respectively, at autoregressive context \(i\). When the complete
predictive distributions are available, OPD can minimize the forward KL,
reverse KL, or a generalized Jensen--Shannon divergence
\citep{hinton2015distilling,agarwal2024onpolicy,gu2024minillm}:
\begin{equation*}
\begin{aligned}
 \mathcal L_{\mathrm{FKL}}(\theta)
 &=
 \mathbb E_i\!\left[
 D_{\mathrm{KL}}(q_i\|p_{\theta,i})
 \right],\\
 \mathcal L_{\mathrm{RKL}}(\theta)
 &=
 \mathbb E_i\!\left[
 D_{\mathrm{KL}}(p_{\theta,i}\|q_i)
 \right],\\
 \mathcal L_{\mathrm{JSD}}^{(\alpha)}(\theta)
 &=
 \mathbb E_i\!\left[
 \alpha D_{\mathrm{KL}}(q_i\|m_i^{(\alpha)})
 +(1-\alpha)D_{\mathrm{KL}}(p_{\theta,i}\|m_i^{(\alpha)})
 \right],\\
 m_i^{(\alpha)}
 &=
 \alpha q_i+(1-\alpha)p_{\theta,i}.
\end{aligned}
\end{equation*}
Forward KL gives the conventional soft-target objective and emphasizes
coverage of teacher-supported probability mass. Reverse KL instead penalizes
student probability assigned to teacher-disfavored regions and therefore
typically exhibits more mode-seeking behavior. Generalized JSD compares both
models against a mixture distribution, with \(\alpha=\tfrac12\) recovering
the standard symmetric JSD \citep{kullback1951information,lin1991divergence}.

\paragraph{Top-\(K\) distribution matching.}

Top-\(K\) OPD retains distribution-level supervision over a restricted local
support. Common choices include a student-selected support
\citep{li2026rethinkingopd,ye2026opcd} and a teacher-selected support
\citep{fu2026revisitingopd}:
\begin{equation*}
\begin{aligned}
 S_{i,p}^{(K)}
 &\triangleq
 \operatorname{TopK}(p_{\theta,i},K),\\
 S_{i,q}^{(K)}
 &\triangleq
 \operatorname{TopK}(q_i,K),\\
 S_i^{(K)}
 &\in
 \left\{S_{i,p}^{(K)},S_{i,q}^{(K)}\right\}.
\end{aligned}
\end{equation*}
For the selected support \(S_i^{(K)}\), define the restricted and
renormalized distributions
\begin{equation*}
\begin{aligned}
 \bar p_{\theta,i}^{S_i^{(K)}}(v)
 &\triangleq
 \frac{
 p_{\theta,i}(v)\mathbf 1\{v\in S_i^{(K)}\}
 }{
 \sum_{u\in S_i^{(K)}}p_{\theta,i}(u)
 },\\
 \bar q_i^{S_i^{(K)}}(v)
 &\triangleq
 \frac{
 q_i(v)\mathbf 1\{v\in S_i^{(K)}\}
 }{
 \sum_{u\in S_i^{(K)}}q_i(u)
 }.
\end{aligned}
\end{equation*}
A representative truncated reverse-KL objective is
\begin{equation*}
\begin{aligned}
 \mathcal L_{\mathrm{TopK\text{-}RKL}}(\theta)
 &=
 \mathbb E_i\!\left[
 D_{\mathrm{KL}}\!\left(
 \bar p_{\theta,i}^{S_i^{(K)}}
 \,\middle\|\,
 \bar q_i^{S_i^{(K)}}
 \right)
 \right]\\
 &=
 \mathbb E_i\!\left[
 \sum_{v\in S_i^{(K)}}
 \bar p_{\theta,i}^{S_i^{(K)}}(v)
 \log
 \frac{
 \bar p_{\theta,i}^{S_i^{(K)}}(v)
 }{
 \bar q_i^{S_i^{(K)}}(v)
 }
 \right].
\end{aligned}
\end{equation*}
Top-\(K\) matching occupies an intermediate point between one-token and
full-vocabulary supervision. It preserves multi-token information at reduced
computational or communication cost, but discards probability mass outside the
selected support and is therefore a truncated, support-dependent approximation
to the full reverse KL.

\paragraph{Sampled-token distillation.}

At a fixed on-policy context, define the teacher--student log-ratio cost
\begin{equation*}
 \delta_i(v)
 \triangleq
 \log p_{\theta,i}(v)-\log q_i(v).
\end{equation*}
The token-level reverse KL can then be written exactly as an expectation over
student-sampled tokens:
\begin{equation*}
\begin{aligned}
 D_{\mathrm{KL}}(p_{\theta,i}\|q_i)
 &=
 \mathbb E_{a_i\sim p_{\theta,i}}
 \left[
 \delta_i(a_i)
 \right]\\
 &=
 \mathbb E_{a_i\sim b_i}
 \left[
 \rho_{\theta,i}(a_i)\delta_i(a_i)
 \right],
 \qquad
 \rho_{\theta,i}(a)
 \triangleq
 \frac{p_{\theta,i}(a)}{b_i(a)}.
\end{aligned}
\end{equation*}
The second equality requires \(p_{\theta,i}\ll b_i\)(support coverage condition). Consequently,
\(\rho_{\theta,i}(a_i)\delta_i(a_i)\) is an importance-weighted
single-sample estimator of the per-context reverse KL. Its score-function
gradient is
\begin{equation*}
\begin{aligned}
 \nabla_\theta
 D_{\mathrm{KL}}(p_{\theta,i}\|q_i)
 =
 \mathbb E_{a_i\sim b_i}
 \left[
 \rho_{\theta,i}(a_i)
 \operatorname{sg}\!\left[\delta_i(a_i)\right]
 \nabla_\theta\log p_{\theta,i}(a_i)
 \right],
\end{aligned}
\end{equation*}
where \(\operatorname{sg}\) denotes stop-gradient. This connection permits
sampled-token distillation to be implemented with policy-gradient or
importance-weighted policy-optimization machinery
\citep{gu2024minillm,lu2025onpolicydistillation,oh2026kl}.
Compared with full-vocabulary matching, sampled-token supervision requires
only the teacher probability of the realized token, but has higher
Monte Carlo variance and uses less information from the teacher distribution.

\paragraph{From the clipped OPID objective to its unclipped skill
surrogate.}
Let \(i=(\tau,t,\ell)\) index a valid rollout-token position, and let
\(\nu_b\) denote the distribution over valid token positions induced by
rollouts collected from the behavior policy. Given position \(i\), the
observed token \(a_i\) is sampled from \(b_i\). Recall from
Eq.~\ref{eq:gap_ratio} that
\[
 \Delta_i(v)
 =
 \log q_i(v)-\log b_i(v),
 \qquad
 \rho_{\theta,i}(v)
 =
 \frac{p_{\theta,i}(v)}{b_i(v)},
\]
where \(b_i\), \(q_i\), and the resulting advantages are detached during
the policy update. The skill advantage of a sampled token is
\begin{equation*}
 A_i^{\mathrm{skill}}(a_i)
 =
 \Delta_i(a_i).
\end{equation*}

The complete OPID advantage combines the outcome and skill signals:
\begin{equation*}
 A_i^{\mathrm{OPID}}(a_i)
 =
 A_i^{\mathrm{ep}}
 +
 \lambda_{\mathrm{skill}}\Delta_i(a_i).
\end{equation*}
Accordingly, the implemented clipped policy loss is
\begin{equation*}
\begin{aligned}
 \mathcal L_{\mathrm{policy}}(\theta)
 =
 -
 \mathbb E_{\substack{i\sim\nu_b\\a_i\sim b_i}}
 \Big[
 \min\Big(
   &\rho_{\theta,i}(a_i)
    A_i^{\mathrm{OPID}}(a_i),
 \\
   &\operatorname{clip}
   \bigl(
     \rho_{\theta,i}(a_i),
     1-\epsilon,
     1+\epsilon
   \bigr)
   A_i^{\mathrm{OPID}}(a_i)
 \Big)
 \Big].
\end{aligned}
\end{equation*}
In a realized rollout batch, this expectation is implemented as an
empirical average over the observed valid tokens \(a_i\).

Because PPO clipping is applied after the outcome and skill advantages
have been combined, the clipped objective does not in general decompose
into independently clipped outcome and skill losses. To isolate the
skill-distillation signal studied below, we therefore consider the
corresponding unclipped policy surrogate:
\begin{equation*}
 \mathcal L_{\mathrm{policy}}^{\mathrm{unclip}}(\theta)
 \triangleq
 -
 \mathbb E_{\substack{i\sim\nu_b\\a_i\sim b_i}}
 \left[
   \rho_{\theta,i}(a_i)
   A_i^{\mathrm{OPID}}(a_i)
 \right].
\end{equation*}
Unlike the clipped objective, this loss decomposes exactly as
\begin{equation*}
 \mathcal L_{\mathrm{policy}}^{\mathrm{unclip}}(\theta)
 =
 \mathcal L_{\mathrm{ep}}^{\mathrm{unclip}}(\theta)
 +
 \mathcal L_{\mathrm{skill}}^{\mathrm{unclip}}(\theta),
\end{equation*}
where
\begin{equation*}
 \mathcal L_{\mathrm{ep}}^{\mathrm{unclip}}(\theta)
 \triangleq
 -
 \mathbb E_{\substack{i\sim\nu_b\\a_i\sim b_i}}
 \left[
   \rho_{\theta,i}(a_i)A_i^{\mathrm{ep}}
 \right]
\end{equation*}
and
\begin{equation}
 \mathcal L_{\mathrm{skill}}^{\mathrm{unclip}}(\theta)
 \triangleq
 -
 \lambda_{\mathrm{skill}}
 \mathbb E_{\substack{i\sim\nu_b\\a_i\sim b_i}}
 \left[
   \rho_{\theta,i}(a_i)\Delta_i(a_i)
 \right].
\label{eq:opid_unclipped_skill_loss}
\end{equation}

Equation~\ref{eq:opid_unclipped_skill_loss} is the skill-distillation
loss analyzed in the next subsection. Although it is defined through the
unclipped surrogate, it characterizes the local skill-induced update of
the implemented PPO loss. In particular, let
\(\theta_0=\theta_{\mathrm{old}}\), so that
\(p_{\theta_0,i}=b_i\) and
\(\rho_{\theta_0,i}(a)=1\). Since \(1\) lies in the interior of the
clipping interval, the clipped and unclipped objectives have the same
value and gradient at the behavior policy:
\begin{equation*}
\begin{aligned}
 \mathcal L_{\mathrm{policy}}^{\mathrm{clip}}(\theta_0)
 &=
 \mathcal L_{\mathrm{policy}}^{\mathrm{unclip}}(\theta_0),
 \\
 \left.
 \nabla_\theta
 \mathcal L_{\mathrm{policy}}^{\mathrm{clip}}(\theta)
 \right|_{\theta=\theta_0}
 &=
 \left.
 \nabla_\theta
 \mathcal L_{\mathrm{policy}}^{\mathrm{unclip}}(\theta)
 \right|_{\theta=\theta_0}
 \\
 &=
 \left.
 \nabla_\theta
 \mathcal L_{\mathrm{ep}}^{\mathrm{unclip}}(\theta)
 \right|_{\theta=\theta_0}
 +
 \left.
 \nabla_\theta
 \mathcal L_{\mathrm{skill}}^{\mathrm{unclip}}(\theta)
 \right|_{\theta=\theta_0}.
\end{aligned}
\end{equation*}
Thus,
\(\mathcal L_{\mathrm{skill}}^{\mathrm{unclip}}\) is exactly the
skill-induced component of the first-order PPO update around the behavior
policy. Away from this local region, clipping couples the outcome and
skill signals through the sign of their combined advantage, and the
unclipped decomposition no longer describes the complete clipped
objective globally.

\subsection{The Unclipped OPID Skill Loss as a Relative-KL Surrogate}
\label{app:teacher_advantage}

We now analyze the unclipped skill-distillation loss introduced in
Eq.~\ref{eq:opid_unclipped_skill_loss}. Let \(\nu_b\) denote the
distribution over valid token positions induced by rollouts collected from
the behavior policy. Throughout this subsection, the rollout histories,
routed skills, and the corresponding distributions \(b_i\) and \(q_i\) are
detached and held fixed during the policy update.

We assume the common-support condition
\begin{equation*}
 p_{\theta,i}\ll b_i
 \qquad\text{and}\qquad
 p_{\theta,i}\ll q_i
\end{equation*}
for every \(i\) in the support of \(\nu_b\). This condition is satisfied by
standard softmax language models with finite logits.

Recall that
\begin{equation*}
\begin{aligned}
 \Delta_i(v)
 &\triangleq
 \log q_i(v)-\log b_i(v),
 &
 \rho_{\theta,i}(v)
 &\triangleq
 \frac{p_{\theta,i}(v)}{b_i(v)}.
\end{aligned}
\end{equation*}
The unclipped OPID skill loss is
\begin{equation}
 \mathcal L_{\mathrm{skill}}^{\mathrm{unclip}}(\theta)
 \triangleq
 -
 \lambda_{\mathrm{skill}}
 \mathbb E_{\substack{i\sim\nu_b\\a\sim b_i}}
 \left[
   \rho_{\theta,i}(a)\Delta_i(a)
 \right].
\label{eq:opid_skill_loss_recalled}
\end{equation}
In a realized rollout batch, this expectation is approximated by the
empirical average over the observed valid tokens. The expectation notation in
Eq.~\ref{eq:opid_skill_loss_recalled} makes the rollout-time token sampling
law explicit for the theoretical analysis.

Define the behavior-relative KL and the student--teacher reverse-KL loss as
\begin{equation*}
\begin{aligned}
 \mathcal D_b(\theta)
 &\triangleq
 \mathbb E_{i\sim\nu_b}
 \left[
   D_{\mathrm{KL}}
   \left(
     p_{\theta,i}\|b_i
   \right)
 \right],
 \\
 \mathcal L_{\mathrm{RKL}}(\theta)
 &\triangleq
 \mathbb E_{i\sim\nu_b}
 \left[
   D_{\mathrm{KL}}
   \left(
     p_{\theta,i}\|q_i
   \right)
 \right].
\end{aligned}
\end{equation*}

\begin{proposition}[Exact relative-KL decomposition]
\label{prop:relative_kl}
Under the assumptions above, for every admissible \(\theta\),
\begin{equation}
 \mathcal L_{\mathrm{skill}}^{\mathrm{unclip}}(\theta)
 =
 \lambda_{\mathrm{skill}}
 \left[
   \mathcal L_{\mathrm{RKL}}(\theta)
   -
   \mathcal D_b(\theta)
 \right].
\label{eq:relative_kl_decomposition}
\end{equation}
Let \(\theta_0=\theta_{\mathrm{old}}\) and suppose that
\(p_{\theta_0,i}=b_i\) for every \(i\). Then
\begin{align}
 \mathcal L_{\mathrm{skill}}^{\mathrm{unclip}}(\theta_0)
 &=
 \lambda_{\mathrm{skill}}
 \mathcal L_{\mathrm{RKL}}(\theta_0),
\label{eq:value_equivalence}
\\
 \left.
 \nabla_\theta
 \mathcal L_{\mathrm{skill}}^{\mathrm{unclip}}(\theta)
 \right|_{\theta=\theta_0}
 &=
 \lambda_{\mathrm{skill}}
 \left.
 \nabla_\theta
 \mathcal L_{\mathrm{RKL}}(\theta)
 \right|_{\theta=\theta_0}
\label{eq:gradient_equivalence}
\\
 &=
 -
 \lambda_{\mathrm{skill}}
 \mathbb E_{\substack{i\sim\nu_b\\a\sim b_i}}
 \left[
   \Delta_i(a)
   \left.
   \nabla_\theta
   \log p_{\theta,i}(a)
   \right|_{\theta=\theta_0}
 \right].
\label{eq:sampled_local_gradient}
\end{align}
\end{proposition}

\begin{proof}
Fix a valid token position \(i\). By the common-support assumption and a
change of measure from \(b_i\) to \(p_{\theta,i}\),
\begin{align*}
 &-
 \lambda_{\mathrm{skill}}
 \mathbb E_{a\sim b_i}
 \left[
   \rho_{\theta,i}(a)\Delta_i(a)
 \right]
 \\
 &\quad=
 -
 \lambda_{\mathrm{skill}}
 \sum_{v\in\mathcal V}
 b_i(v)
 \frac{p_{\theta,i}(v)}{b_i(v)}
 \bigl(
   \log q_i(v)-\log b_i(v)
 \bigr)
 \\
 &\quad=
 \lambda_{\mathrm{skill}}
 \sum_{v\in\mathcal V}
 p_{\theta,i}(v)
 \bigl(
   \log b_i(v)-\log q_i(v)
 \bigr).
\end{align*}
Adding and subtracting \(\log p_{\theta,i}(v)\) inside the summand gives
\begin{align*}
 &\lambda_{\mathrm{skill}}
 \sum_{v\in\mathcal V}
 p_{\theta,i}(v)
 \left[
   \log\frac{p_{\theta,i}(v)}{q_i(v)}
   -
   \log\frac{p_{\theta,i}(v)}{b_i(v)}
 \right]
 \\
 &\quad=
 \lambda_{\mathrm{skill}}
 \left[
   D_{\mathrm{KL}}(p_{\theta,i}\|q_i)
   -
   D_{\mathrm{KL}}(p_{\theta,i}\|b_i)
 \right].
\end{align*}
Averaging over \(i\sim\nu_b\) proves
Eq.~\ref{eq:relative_kl_decomposition}.

At \(\theta_0\), \(p_{\theta_0,i}=b_i\), and hence
\begin{equation*}
 \mathcal D_b(\theta_0)=0.
\end{equation*}
This proves Eq.~\ref{eq:value_equivalence}. Moreover,
\(\mathcal D_b\) is differentiable and attains its global minimum at
\(\theta_0\), so
\begin{equation*}
 \left.
 \nabla_\theta\mathcal D_b(\theta)
 \right|_{\theta=\theta_0}
 =0.
\end{equation*}
Differentiating Eq.~\ref{eq:relative_kl_decomposition} therefore proves
Eq.~\ref{eq:gradient_equivalence}.

Finally, because \(b_i\), \(q_i\), and \(\Delta_i\) are detached,
\begin{equation*}
\begin{aligned}
 \nabla_\theta
 \mathcal L_{\mathrm{skill}}^{\mathrm{unclip}}(\theta)
 =
 -
 \lambda_{\mathrm{skill}}
 \mathbb E_{\substack{i\sim\nu_b\\a\sim b_i}}
 \left[
   \rho_{\theta,i}(a)\Delta_i(a)
   \nabla_\theta\log p_{\theta,i}(a)
 \right].
\end{aligned}
\end{equation*}
Substituting
\(\rho_{\theta_0,i}(a)=1\) proves
Eq.~\ref{eq:sampled_local_gradient}.
\end{proof}

\begin{remark}[Why the OPID skill loss is not the direct reverse-KL loss]
\label{rem:not_direct_kl}
The scaled direct reverse-KL loss is
\begin{equation*}
 \lambda_{\mathrm{skill}}
 \mathcal L_{\mathrm{RKL}}(\theta)
 =
 \lambda_{\mathrm{skill}}
 \mathbb E_{\substack{i\sim\nu_b\\v\sim p_{\theta,i}}}
 \left[
   \log p_{\theta,i}(v)-\log q_i(v)
 \right],
\end{equation*}
whereas the OPID skill loss can be written as
\begin{equation*}
 \mathcal L_{\mathrm{skill}}^{\mathrm{unclip}}(\theta)
 =
 \lambda_{\mathrm{skill}}
 \mathbb E_{\substack{i\sim\nu_b\\v\sim p_{\theta,i}}}
 \left[
   \log b_i(v)-\log q_i(v)
 \right].
\end{equation*}
The two expressions differ because the denominator in the detached
teacher advantage is the rollout policy \(b_i\), rather than the current
student \(p_{\theta,i}\). Importance weighting changes the sampling
distribution from \(b_i\) to \(p_{\theta,i}\), but it does not replace
\(\log b_i\) by \(\log p_{\theta,i}\). Consequently,
\begin{equation*}
 \mathcal L_{\mathrm{skill}}^{\mathrm{unclip}}(\theta)
 -
 \lambda_{\mathrm{skill}}
 \mathcal L_{\mathrm{RKL}}(\theta)
 =
 -
 \lambda_{\mathrm{skill}}\mathcal D_b(\theta).
\end{equation*}
Thus, the OPID skill loss is an exact relative-KL loss and only a local
surrogate for direct student--teacher reverse-KL matching.

This distinction also changes the global optimum. For example, consider
\[
 b=\left(\frac12,\frac12\right),
 \qquad
 q=\left(\frac34,\frac14\right).
\]
For a categorical distribution \(p=(p_1,p_2)\),
\begin{equation*}
 \frac{
 \mathcal L_{\mathrm{skill}}^{\mathrm{unclip}}(p)
 }{
 \lambda_{\mathrm{skill}}
 }
 =
 p_1\log\frac{2}{3}
 +
 p_2\log 2,
\end{equation*}
which is linear in \(p\) and whose infimum is approached by concentrating
all probability mass on the first token. In contrast, the direct reverse-KL
loss is uniquely minimized at \(p=q\). Therefore, the two losses cannot be
identified globally.
\end{remark}

\begin{corollary}[First-order tightness around the behavior policy]
\label{cor:first_order_tightness}
Assume that \(p_{\theta,i}\) is twice continuously differentiable in a
neighborhood of \(\theta_0\). For \(\delta\to0\),
\begin{equation}
\begin{aligned}
 \mathcal L_{\mathrm{skill}}^{\mathrm{unclip}}
 (\theta_0+\delta)
 &=
 \lambda_{\mathrm{skill}}
 \mathcal L_{\mathrm{RKL}}(\theta_0+\delta)
 \\
 &\quad-
 \frac{\lambda_{\mathrm{skill}}}{2}
 \delta^\top F_b\delta
 +
 o(\|\delta\|^2),
\end{aligned}
\label{eq:second_order_gap}
\end{equation}
where
\begin{equation*}
\begin{aligned}
 F_b
 &\triangleq
 \mathbb E_{\substack{i\sim\nu_b\\v\sim b_i}}
 \left[
   s_i(v)s_i(v)^\top
 \right],
 \\
 s_i(v)
 &\triangleq
 \left.
 \nabla_\theta\log p_{\theta,i}(v)
 \right|_{\theta=\theta_0}
\end{aligned}
\end{equation*}
is the behavior-policy Fisher information averaged over rollout contexts.
\end{corollary}

\begin{proof}
By Proposition~\ref{prop:relative_kl}, the discrepancy between the scaled
reverse-KL loss and the OPID skill loss is exactly
\(\lambda_{\mathrm{skill}}\mathcal D_b(\theta)\). The standard local
expansion of relative entropy around its reference distribution gives
\begin{equation*}
 \mathcal D_b(\theta_0+\delta)
 =
 \frac12
 \delta^\top F_b\delta
 +
 o(\|\delta\|^2).
\end{equation*}
Substituting this expansion into
Eq.~\ref{eq:relative_kl_decomposition} proves
Eq.~\ref{eq:second_order_gap}.
\end{proof}

Equation~\ref{eq:second_order_gap} gives the precise sense in which the
OPID skill loss is locally equivalent to reverse-KL distillation. At the
behavior policy, the two losses have the same value and gradient after
accounting for the factor \(\lambda_{\mathrm{skill}}\), while their
discrepancy is second order in the policy displacement.

\begin{corollary}[Exact recovery under a matching behavior-KL penalty]
\label{cor:matching_kl}
Consider the regularized auxiliary loss
\begin{equation}
 \mathcal L_{\mathrm{aux}}(\theta)
 \triangleq
 \mathcal L_{\mathrm{skill}}^{\mathrm{unclip}}(\theta)
 +
 \beta\mathcal D_b(\theta).
\label{eq:regularized_auxiliary_loss}
\end{equation}
Then
\begin{equation}
 \mathcal L_{\mathrm{aux}}(\theta)
 =
 \lambda_{\mathrm{skill}}
 \mathcal L_{\mathrm{RKL}}(\theta)
 +
 \left(
   \beta-\lambda_{\mathrm{skill}}
 \right)
 \mathcal D_b(\theta).
\label{eq:regularized_decomposition}
\end{equation}
In particular, if
\(\beta=\lambda_{\mathrm{skill}}\), then
\begin{equation*}
 \mathcal L_{\mathrm{aux}}(\theta)
 =
 \lambda_{\mathrm{skill}}
 \mathcal L_{\mathrm{RKL}}(\theta)
\end{equation*}
for every admissible \(\theta\).
\end{corollary}

\begin{proof}
Substitute Eq.~\ref{eq:relative_kl_decomposition} into
Eq.~\ref{eq:regularized_auxiliary_loss} and collect the coefficients of
\(\mathcal D_b(\theta)\).
\end{proof}

The exact cancellation in Corollary~\ref{cor:matching_kl} requires both
(i) a KL penalty to the same behavior distribution \(b_i\), evaluated under
the ordinary context, and
(ii) the matching coefficient
\(\beta=\lambda_{\mathrm{skill}}\).
A KL penalty to a different reference distribution, or a different
coefficient, leaves the residual behavior-relative term in
Eq.~\ref{eq:regularized_decomposition} and is therefore not exactly
equivalent to direct student--teacher reverse-KL distillation.

\paragraph{Relation to the implemented PPO-clipped loss.}

The decomposition in Proposition~\ref{prop:relative_kl} applies exactly to
the unclipped skill loss
\(\mathcal L_{\mathrm{skill}}^{\mathrm{unclip}}\).
In the implemented OPID objective, PPO clipping is applied to the combined
advantage
\[
 A_i^{\mathrm{OPID}}
 =
 A_i^{\mathrm{ep}}
 +
 \lambda_{\mathrm{skill}}\Delta_i(a_i),
\]
so the complete clipped loss does not globally decompose into independently
clipped outcome and skill losses.

Nevertheless, at
\(\theta_0=\theta_{\mathrm{old}}\),
\[
 \rho_{\theta_0,i}(a)=1.
\]
Since \(1\) lies in the interior of
\([1-\epsilon,1+\epsilon]\) for \(\epsilon>0\), the clipped and unclipped
policy losses have the same value and first derivative at the behavior
policy:
\begin{equation*}
\begin{aligned}
 \mathcal L_{\mathrm{policy}}^{\mathrm{clip}}(\theta_0)
 &=
 \mathcal L_{\mathrm{policy}}^{\mathrm{unclip}}(\theta_0),
 \\
 \left.
 \nabla_\theta
 \mathcal L_{\mathrm{policy}}^{\mathrm{clip}}(\theta)
 \right|_{\theta=\theta_0}
 &=
 \left.
 \nabla_\theta
 \mathcal L_{\mathrm{policy}}^{\mathrm{unclip}}(\theta)
 \right|_{\theta=\theta_0}
 \\
 &=
 \left.
 \nabla_\theta
 \mathcal L_{\mathrm{ep}}^{\mathrm{unclip}}(\theta)
 \right|_{\theta=\theta_0}
 +
 \left.
 \nabla_\theta
 \mathcal L_{\mathrm{skill}}^{\mathrm{unclip}}(\theta)
 \right|_{\theta=\theta_0}.
\end{aligned}
\end{equation*}
Therefore,
Eq.~\ref{eq:gradient_equivalence} characterizes the skill-induced
component of the local PPO update. Once the policy ratio reaches a clipping
boundary, however, clipping couples the outcome and skill signals through
the sign of their combined advantage, and the exact relative-KL
decomposition no longer applies to the complete clipped objective.

\begin{corollary}[Non-degenerate token-level signal under reward ties]
\label{cor:dense_signal}
Fix one context \(i\), and parameterize
\(p_i=\operatorname{softmax}(z_i)\) using free categorical logits. Define
the corresponding full-action skill loss as
\begin{equation*}
 \mathcal L_{\mathrm{skill},i}^{\mathrm{unclip}}(z_i)
 \triangleq
 -
 \lambda_{\mathrm{skill}}
 \sum_{v\in\mathcal V}
 p_i(v)\Delta_i(v).
\end{equation*}
For \(\lambda_{\mathrm{skill}}>0\),
\begin{equation}
 \frac{
   \partial
   \mathcal L_{\mathrm{skill},i}^{\mathrm{unclip}}
 }{
   \partial z_i(v)
 }
 =
 -
 \lambda_{\mathrm{skill}}
 p_i(v)
 \left(
   \Delta_i(v)
   -
   \mathbb E_{u\sim p_i}[\Delta_i(u)]
 \right).
\label{eq:logit_gradient}
\end{equation}
At \(p_i=b_i\) with full support, the gradient in
Eq.~\ref{eq:logit_gradient} is zero for every \(v\) if and only if
\(q_i=b_i\).
\end{corollary}

\begin{proof}
Using
\[
 \frac{\partial p_i(u)}{\partial z_i(v)}
 =
 p_i(u)
 \left(
   \mathbf 1\{u=v\}-p_i(v)
 \right),
\]
we obtain
\begin{align*}
 \frac{
   \partial
   \mathcal L_{\mathrm{skill},i}^{\mathrm{unclip}}
 }{
   \partial z_i(v)
 }
 &=
 -
 \lambda_{\mathrm{skill}}
 \sum_u
 \Delta_i(u)
 p_i(u)
 \left(
   \mathbf 1\{u=v\}-p_i(v)
 \right)
 \\
 &=
 -
 \lambda_{\mathrm{skill}}
 p_i(v)\Delta_i(v)
 +
 \lambda_{\mathrm{skill}}
 p_i(v)
 \sum_u p_i(u)\Delta_i(u),
\end{align*}
which proves Eq.~\ref{eq:logit_gradient}.

Suppose that \(p_i=b_i\), \(b_i(v)>0\) for every \(v\), and the derivative
is zero for every \(v\). Since
\(\lambda_{\mathrm{skill}}>0\), it follows that
\(\Delta_i(v)\) is constant over the vocabulary. Hence
\[
 q_i(v)=e^c b_i(v)
\]
for some constant \(c\). Normalization of \(q_i\) and \(b_i\) implies
\(e^c=1\), and therefore \(q_i=b_i\). The converse is immediate.
\end{proof}

Corollary~\ref{cor:dense_signal} is a per-context logit statement. It shows
that even when group-relative outcome advantages vanish because all sampled
trajectories receive tied rewards, a nontrivial skill-conditioned teacher
still supplies a token-level learning signal whenever \(q_i\neq b_i\).
With shared neural parameters, gradients from different contexts may still
cancel; the result does not claim that the aggregate parameter gradient must
be nonzero.

\subsection{On-Policy Occupancy Matching for Distillation}
\label{app:occupancy_matching}

Recall that \(\nu_b\) denotes the distribution over valid token positions
induced by rollouts collected from the behavior policy. Let \(d_b\) denote
the corresponding distribution over ordinary autoregressive contexts
\(c_i\), i.e., the context marginal induced by \(i\sim\nu_b\). For an
arbitrary data-collection policy \(\mu\), let \(d_\mu\) denote the analogous
context distribution.

We define total variation as
\begin{equation*}
 \operatorname{TV}(P,Q)
 \triangleq
 \sup_A |P(A)-Q(A)|
 =
 \frac{1}{2}
 \int
 \left|
   \mathrm dP-\mathrm dQ
 \right|.
\end{equation*}

The following result isolates the effect of changing only the distribution
of ordinary autoregressive contexts. It applies to both nonnegative
distillation losses and signed surrogate losses.

\begin{proposition}[On-policy occupancy matching]
\label{prop:occupancy_matching}
Let
\(\ell_\theta:\mathcal C\rightarrow[m_\ell,M_\ell]\)
be a measurable per-context loss, where
\(-\infty<m_\ell<M_\ell<+\infty\). Then
\begin{equation}
\begin{aligned}
 &
 \left|
   \mathbb E_{c\sim d_b}
   \left[
     \ell_\theta(c)
   \right]
   -
   \mathbb E_{c\sim d_\mu}
   \left[
     \ell_\theta(c)
   \right]
 \right|
 \\
 &\qquad\le
 \left(
   M_\ell-m_\ell
 \right)
 \operatorname{TV}(d_b,d_\mu)
 \\
 &\qquad\le
 \left(
   M_\ell-m_\ell
 \right)
 \sqrt{
   \frac{1}{2}
   D_{\mathrm{KL}}(d_b\|d_\mu)
 }.
\end{aligned}
\label{eq:occupancy_bound}
\end{equation}
In particular, if \(d_\mu=d_b\), then the context-occupancy mismatch is
exactly zero.
\end{proposition}

\begin{proof}
Define
\begin{equation*}
 f_\theta(c)
 \triangleq
 \frac{
   \ell_\theta(c)-m_\ell
 }{
   M_\ell-m_\ell
 }.
\end{equation*}
Then \(0\le f_\theta(c)\le 1\). By the variational characterization of
total variation over measurable functions with range in \([0,1]\),
\begin{equation*}
 \left|
   \mathbb E_{d_b}[f_\theta]
   -
   \mathbb E_{d_\mu}[f_\theta]
 \right|
 \le
 \operatorname{TV}(d_b,d_\mu).
\end{equation*}
Multiplying both sides by \(M_\ell-m_\ell\) proves the first inequality in
Eq.~\ref{eq:occupancy_bound}. The second inequality follows from
Pinsker's inequality. If \(d_b\) is not absolutely continuous with respect
to \(d_\mu\), then
\(D_{\mathrm{KL}}(d_b\|d_\mu)=+\infty\), and the inequality remains valid
in the extended-real sense. Setting \(d_\mu=d_b\) proves the final
statement.
\end{proof}

For example, Proposition~\ref{prop:occupancy_matching} can be applied to
the per-context reverse-KL loss
\begin{equation*}
 \ell_{\mathrm{RKL},\theta}(c_i)
 \triangleq
 D_{\mathrm{KL}}
 \left(
   p_{\theta,i}\|q_i
 \right),
\end{equation*}
which is the distribution-matching loss locally approximated by the OPID
skill update. It can also be applied to a bounded version of the signed
per-context OPID skill loss
\begin{equation}
\begin{aligned}
 \ell_{\mathrm{skill},\theta}^{\mathrm{unclip}}(c_i)
 &\triangleq
 -
 \lambda_{\mathrm{skill}}
 \mathbb E_{a\sim b_i}
 \left[
   \rho_{\theta,i}(a)\Delta_i(a)
 \right]
 \\
 &=
 \lambda_{\mathrm{skill}}
 \left[
   D_{\mathrm{KL}}
   \left(
     p_{\theta,i}\|q_i
   \right)
   -
   D_{\mathrm{KL}}
   \left(
     p_{\theta,i}\|b_i
   \right)
 \right].
\end{aligned}
\label{eq:per_context_opid_skill_loss}
\end{equation}
Because the loss in Eq.~\ref{eq:per_context_opid_skill_loss} is signed
and need not be uniformly bounded for arbitrary probability distributions,
applying Proposition~\ref{prop:occupancy_matching} to it requires an
explicit bounded-range condition, such as probability flooring, log-ratio
clipping, or restriction to a compact parameter neighborhood. More general
versions can instead be obtained under appropriate moment or tail
conditions.

Proposition~\ref{prop:occupancy_matching} controls only the mismatch in the
outer distribution of ordinary autoregressive contexts. It assumes that
the same per-context loss map is evaluated under \(d_b\) and \(d_\mu\).
It does not by itself control changes in the hindsight skill, the routed
teacher \(q_i\), or other trajectory-dependent quantities that may also
change with the data-collection policy.
\subsection{Critical-First Hierarchical Routing}
\label{app:routing}

We next formalize how the episode-level and step-level skills determine the
detached teacher \(q_i\) used in
\(\mathcal L_{\mathrm{skill}}^{\mathrm{unclip}}\).

Let \(q_i^\star\) denote an ideal privileged teacher at token position
\(i\). Let \(q_i^{\mathrm{ep}}\) and \(q_i^{\mathrm{step}}\) denote the
teachers induced by the episode-level and step-level skills, respectively.
Let
\begin{equation*}
 z_i^\star\in\{0,1\}
\end{equation*}
be an oracle criticality indicator, where \(z_i^\star=1\) means that the
step-level teacher is the appropriate specialized teacher. The analyzer
prediction is
\begin{equation*}
 \widehat z_i
 \triangleq
 \mathbf 1\{t\in C_\tau\}.
\end{equation*}

The critical-first routing rule defines
\begin{equation}
 q_i^{\mathrm{route}}
 \triangleq
 \widehat z_i q_i^{\mathrm{step}}
 +
 \left(
   1-\widehat z_i
 \right)
 q_i^{\mathrm{ep}},
 \qquad
 q_i\equiv q_i^{\mathrm{route}}.
\label{eq:routed_teacher}
\end{equation}
Thus, the \(q_i\) appearing in the OPID skill advantage
\(\Delta_i(v)=\log q_i(v)-\log b_i(v)\) is precisely the routed teacher in
Eq.~\ref{eq:routed_teacher}.

Measure the approximation errors of the two candidate teachers by
\begin{equation}
\begin{aligned}
 \mathcal E_i^{\mathrm{ep}}
 &\triangleq
 D_{\mathrm{KL}}
 \left(
   q_i^\star\|q_i^{\mathrm{ep}}
 \right),
 \\
 \mathcal E_i^{\mathrm{step}}
 &\triangleq
 D_{\mathrm{KL}}
 \left(
   q_i^\star\|q_i^{\mathrm{step}}
 \right),
 \\
 \mathcal E_i^{\mathrm{route}}
 &\triangleq
 D_{\mathrm{KL}}
 \left(
   q_i^\star\|q_i^{\mathrm{route}}
 \right).
\end{aligned}
\label{eq:routing_errors}
\end{equation}
Because the routing decision is hard,
\begin{equation*}
 \mathcal E_i^{\mathrm{route}}
 =
 \widehat z_i
 \mathcal E_i^{\mathrm{step}}
 +
 \left(
   1-\widehat z_i
 \right)
 \mathcal E_i^{\mathrm{ep}}.
\end{equation*}

\begin{proposition}[Routing optimality and detector-error regret]
\label{prop:routing_regret}
Assume that the episode-level and step-level teachers specialize according
to the oracle criticality label:
\begin{equation}
\begin{aligned}
 z_i^\star=1
 &\implies
 \mathcal E_i^{\mathrm{step}}
 \le
 \mathcal E_i^{\mathrm{ep}},
 \\
 z_i^\star=0
 &\implies
 \mathcal E_i^{\mathrm{ep}}
 \le
 \mathcal E_i^{\mathrm{step}}.
\end{aligned}
\label{eq:routing_specialization}
\end{equation}
Then, pointwise,
\begin{equation}
\begin{aligned}
 \mathcal E_i^{\mathrm{route}}
 &=
 \min
 \left\{
   \mathcal E_i^{\mathrm{ep}},
   \mathcal E_i^{\mathrm{step}}
 \right\}
 \\
 &\quad+
 \mathbf 1
 \left\{
   \widehat z_i\ne z_i^\star
 \right\}
 \left|
   \mathcal E_i^{\mathrm{ep}}
   -
   \mathcal E_i^{\mathrm{step}}
 \right|.
\end{aligned}
\label{eq:routing_regret_identity}
\end{equation}
Consequently, if
\begin{equation}
 \left|
   \mathcal E_i^{\mathrm{ep}}
   -
   \mathcal E_i^{\mathrm{step}}
 \right|
 \le
 \Gamma
\label{eq:routing_gap_bound}
\end{equation}
almost surely under \(i\sim\nu_b\), then
\begin{equation}
\begin{aligned}
 \mathbb E_{i\sim\nu_b}
 \left[
   \mathcal E_i^{\mathrm{route}}
 \right]
 &\le
 \min
 \left\{
   \mathbb E_{i\sim\nu_b}
   \left[
     \mathcal E_i^{\mathrm{ep}}
   \right],
   \mathbb E_{i\sim\nu_b}
   \left[
     \mathcal E_i^{\mathrm{step}}
   \right]
 \right\}
 \\
 &\quad+
 \Gamma
 \Pr_{i\sim\nu_b}
 \left(
   \widehat z_i\ne z_i^\star
 \right).
\end{aligned}
\label{eq:routing_expected_bound}
\end{equation}
Under perfect criticality detection,
\begin{equation*}
 \widehat z_i=z_i^\star
 \qquad
 \text{almost surely},
\end{equation*}
and therefore
\begin{equation*}
 \mathcal E_i^{\mathrm{route}}
 =
 \min
 \left\{
   \mathcal E_i^{\mathrm{ep}},
   \mathcal E_i^{\mathrm{step}}
 \right\}
\end{equation*}
pointwise, with
\begin{equation*}
\begin{aligned}
 \mathbb E_{i\sim\nu_b}
 \left[
   \mathcal E_i^{\mathrm{route}}
 \right]
 \le
 \min
 \left\{
   \mathbb E_{i\sim\nu_b}
   \left[
     \mathcal E_i^{\mathrm{ep}}
   \right],
   \mathbb E_{i\sim\nu_b}
   \left[
     \mathcal E_i^{\mathrm{step}}
   \right]
 \right\}.
\end{aligned}
\end{equation*}
\end{proposition}

\begin{proof}
Consider first the event
\(\widehat z_i=z_i^\star\). Under
Eq.~\ref{eq:routing_specialization}, the routing rule selects a teacher
with the smaller approximation error. Hence
\begin{equation*}
 \mathcal E_i^{\mathrm{route}}
 =
 \min
 \left\{
   \mathcal E_i^{\mathrm{ep}},
   \mathcal E_i^{\mathrm{step}}
 \right\}.
\end{equation*}
The second term in Eq.~\ref{eq:routing_regret_identity} is zero on this
event.

On the event
\(\widehat z_i\ne z_i^\star\), the routing rule selects the nonspecialized
teacher. Its excess error over the oracle choice is exactly
\begin{equation*}
 \left|
   \mathcal E_i^{\mathrm{ep}}
   -
   \mathcal E_i^{\mathrm{step}}
 \right|.
\end{equation*}
This proves Eq.~\ref{eq:routing_regret_identity}.

Taking expectations yields
\begin{equation*}
\begin{aligned}
 \mathbb E_{i\sim\nu_b}
 \left[
   \mathcal E_i^{\mathrm{route}}
 \right]
 &=
 \mathbb E_{i\sim\nu_b}
 \left[
   \min
   \left\{
     \mathcal E_i^{\mathrm{ep}},
     \mathcal E_i^{\mathrm{step}}
   \right\}
 \right]
 \\
 &\quad+
 \mathbb E_{i\sim\nu_b}
 \left[
   \mathbf 1
   \left\{
     \widehat z_i\ne z_i^\star
   \right\}
   \left|
     \mathcal E_i^{\mathrm{ep}}
     -
     \mathcal E_i^{\mathrm{step}}
   \right|
 \right].
\end{aligned}
\end{equation*}
Using
\begin{equation*}
 \mathbb E[\min\{X,Y\}]
 \le
 \min\{\mathbb E[X],\mathbb E[Y]\}
\end{equation*}
and Eq.~\ref{eq:routing_gap_bound} proves
Eq.~\ref{eq:routing_expected_bound}. The perfect-detection statements
follow by setting
\(\Pr_{i\sim\nu_b}(\widehat z_i\ne z_i^\star)=0\).
\end{proof}

Proposition~\ref{prop:routing_regret} separates the two requirements behind
critical-first routing: teacher specialization and criticality-detection
accuracy. Under specialization, perfect detection recovers the oracle
pointwise choice between the two candidate teachers. With imperfect
detection, the excess teacher-approximation error is controlled jointly by
the detector error probability and the difference between the two candidate
teacher errors.

The criterion in Eq.~\ref{eq:routing_errors} measures the quality of a
candidate teacher relative to \(q_i^\star\). It is distinct from the
student--teacher reverse-KL loss
\(D_{\mathrm{KL}}(p_{\theta,i}\|q_i)\) appearing in
\(\mathcal L_{\mathrm{RKL}}\). Therefore, without additional assumptions
relating the candidate teachers' likelihood ratios, the routing result
should not be interpreted as a direct upper bound on
\(\mathcal L_{\mathrm{RKL}}\).

\subsection{Summary}
\label{app:theory_scope}

Proposition~\ref{prop:relative_kl} analyzes the unclipped skill component
of the OPID policy loss:
\begin{equation*}
 \mathcal L_{\mathrm{skill}}^{\mathrm{unclip}}(\theta)
 =
 -
 \lambda_{\mathrm{skill}}
 \mathbb E_{\substack{i\sim\nu_b\\a\sim b_i}}
 \left[
   \rho_{\theta,i}(a)\Delta_i(a)
 \right].
\end{equation*}
Conditioned on fixed rollout histories, routed skills, and detached
distributions \(b_i\) and \(q_i\), this loss has the exact decomposition
\begin{equation*}
 \mathcal L_{\mathrm{skill}}^{\mathrm{unclip}}(\theta)
 =
 \lambda_{\mathrm{skill}}
 \left[
   \mathcal L_{\mathrm{RKL}}(\theta)
   -
   \mathcal D_b(\theta)
 \right].
\end{equation*}

Proposition~\ref{prop:occupancy_matching}  shows that collecting the ordinary autoregressive
contexts on policy eliminates the outer context-distribution mismatch:
when the collection distribution equals the behavior-policy distribution,
\(d_\mu=d_b\), the occupancy term in
Eq.~\ref{eq:occupancy_bound} is zero. 

Proposition~\ref{prop:routing_regret} analyzes how the teacher \(q_i\) is
selected from episode-level and step-level candidates. Under the stated
specialization assumption, critical-first routing recovers the lower-error
candidate under perfect detection, while the degradation under imperfect
detection is controlled by
\begin{equation*}
 \Gamma
 \Pr_{i\sim\nu_b}
 \left(
   \widehat z_i\ne z_i^\star
 \right).
\end{equation*}

Taken together, the three results establish that:

\begin{enumerate}
 \item The unclipped OPID skill loss is an exact relative-KL loss and is
 first-order equivalent to scaled reverse-KL distillation at the behavior
 policy;

 \item On-policy collection removes the mismatch in the outer distribution
 of ordinary autoregressive contexts; and

 \item Critical-first routing approaches the oracle candidate-teacher
 selection when the candidate teachers specialize and the criticality
 detector is accurate.
\end{enumerate}

\section{Additional Experimental Details}\label{app:experimental_details}
This section provides the experimental protocol used for the results in the main paper. We organize the details by datasets, baselines and implementation.

\subsection{Datasets}
\label{app:datasets}

Table~\ref{tab:appendix_datasets_summary} summarizes the datasets used in our experiments. The evaluation covers three agentic domains: embodied reasoning, web navigation, and search-augmented question answering.

\begin{table}[ht!]
\centering
\caption{
Detailed information on the agentic benchmarks.
}
\vspace{0.05in}
\label{tab:appendix_datasets_summary}
\small
\setlength{\tabcolsep}{5pt}
\renewcommand{\arraystretch}{1.10}
\begin{tabularx}{\linewidth}{
    @{}
    >{\raggedright\arraybackslash}p{0.18\linewidth}
    >{\raggedright\arraybackslash}X
    >{\raggedright\arraybackslash}p{0.24\linewidth}
    >{\raggedright\arraybackslash}p{0.17\linewidth}
    @{}
}
\toprule[1.2pt]
\textbf{Domain}
& \textbf{Benchmark}
& \textbf{\#Train Samples}
& \textbf{\#Test Samples} \\
\midrule

Embodied Reasoning
& ALFWorld
& 2,400 
& \begin{tabular}[t]{@{}l@{}}
140 (seen split) \\
134 (unseen split)
\end{tabular} \\

\midrule

Web Navigation
& WebShop
& 2,400 
& 128  \\

\midrule

Search-Augmented QA
& NQ, TriviaQA, PopQA, HotpotQA, 2WikiMultiHopQA, MuSiQue, and Bamboogle
& 19,200 
& 51,713 \\

\bottomrule[1.2pt]
\end{tabularx}
\end{table}

\paragraph{ALFWorld.}
ALFWorld~\citep{shridhar2020alfworld} aligns text-based interaction with the ALFRED household environment. Given a natural-language goal and textual observations, an agent must issue a sequence of admissible actions to complete the task. We report results on six task types: \textit{Pick}, \textit{Look}, \textit{Clean}, \textit{Heat}, \textit{Cool}, and \textit{Pick2}.

\paragraph{WebShop.}
WebShop~\citep{yao2022webshop} is a text-based e-commerce environment in which an agent searches for products, opens product pages, selects attributes, and purchases an item that satisfies a natural-language request. The environment provides both a normalized task-completion score, which assigns partial credit for matching requested attributes, and a binary success signal for exact task completion.

\paragraph{Search-Augmented QA.}
Following the Search-R1 setting~\citep{jin2025searchr1}, we evaluate search-augmented reasoning on Natural Questions~\citep{kwiatkowski2019natural}, TriviaQA~\citep{joshi2017triviaqa}, PopQA~\citep{mallen2023popqa}, HotpotQA~\citep{yang2018hotpotqa}, 2WikiMultiHopQA~\citep{ho2020constructing}, MuSiQue~\citep{trivedi2022musique}, and Bamboogle~\citep{press2023measuring}. In this setting, the agent interacts with the configured search environment before producing a final answer.

\paragraph{Training Data.}
For training, we conduct separate training for each benchmark setting. Specifically, we sample 2,400 training examples from ALFWorld, 2,400 training examples from WebShop, and 19,200 training examples from the search-augmented QA benchmarks.

\subsection{Baselines}\label{app:baselines}
We compare OPID with prompting-only methods, outcome-based reinforcement learning, and self-distillation or skill-distillation variants. Unless explicitly marked with an asterisk, every method is evaluated from the ordinary environment interaction history, without access to skills or any other privileged context. An asterisk therefore denotes validation/test-time access to a natural-language skill; it does not indicate a different backbone or evaluation task.

\paragraph{Prompting-only methods.}
\begin{itemize}
    \item \textit{Vanilla}. This is the original instruction-tuned backbone used without any post-training. The model receives only the standard environment prompt and the interaction history exposed by the environment interface. 
    \item \textit{Skill-Prompt}$^{*}$. This method keeps the \textit{Vanilla} parameters frozen but augments the validation/test context with a retrieved natural-language skill relevant to the current task. Because no gradient update is performed, any improvement comes purely from in-context use of the skill. 
\end{itemize}

\paragraph{Outcome-based reinforcement learning.}
\begin{itemize}
    \item \textit{GRPO}~\citep{shao2024deepseekmath}. Group Relative Policy Optimization is a critic-free policy-gradient method that samples a group of trajectories for each task, assigns each trajectory a scalar outcome reward, and normalizes these rewards within the group to construct relative advantages. In the outcome-only setting used here, every generated token in a trajectory inherits the same sequence-level advantage, and the policy is updated with a clipped importance-ratio objective; no process labels or teacher-derived token-level targets are used. 

    \item \textit{Skill-GRPO}. This variant uses the same group-relative outcome objective as \textit{GRPO}, but makes a task-relevant natural-language skill available to the policy during training rollouts and policy updates. The skill can therefore shape exploration and the trajectories that receive reinforcement. The skill is removed at validation/test time, so this baseline tests whether skill-guided behavior has been absorbed into the model parameters rather than merely followed from the prompt.

    \item \textit{Skill-GRPO}$^{*}$. This method is trained in the same way as \textit{Skill-GRPO}, but retains the skill context at validation/test time. Its train-time and test-time conditioning are consequently matched. 
\end{itemize}

\paragraph{Self-distillation and skill-distillation methods.}
\begin{itemize}
    \item \textit{OPSD}~\citep{zhao2026opsd}. On-Policy Self-Distillation instantiates a student and a teacher from the same underlying model but gives them different conditioning contexts. The student samples trajectories on-policy from the ordinary task context, whereas the teacher additionally receives training-only privileged information, such as a verified solution or an equivalent auxiliary context. For every prefix of the student's own trajectory, the teacher re-scores the next-token distribution and provides a dense token-level target through full-vocabulary or sampled-token distribution matching. Gradients are applied to the student side while the teacher distribution is treated as a stop-gradient target, and the privileged teacher context is absent at inference time.

    \item \textit{GRPO+OPSD}. This is a direct multi-objective combination of the sequence-level \textit{GRPO} loss and the token-level \textit{OPSD} loss. The outcome term reinforces or penalizes complete trajectories according to environment feedback, while the distillation term supplies local guidance at individual token positions. The two losses are simply combined, making this baseline a controlled test of whether naively adding dense self-distillation to outcome-based RL is sufficient.

    \item \textit{Skill-SD}~\citep{wang2026skillsd}. This method adapts self-distillation to multi-turn agent tasks. Completed trajectories are summarized into compact natural-language skills that record successful behaviors, common failure modes, and reusable high-level workflows. During training, a retrieved skill conditions only the teacher branch, while the student continues to generate on-policy trajectories from the plain task prompt; the student must therefore internalize the teacher-side guidance rather than rely on the skill at test time. 

    \item \textit{RLSD}~\citep{yang2026rlsd}. RLSD uses a privileged self-teacher for fine-grained credit assignment without directly optimizing a teacher--student distribution-matching loss. It converts the token-wise teacher--student log-probability gap into a bounded weight that modulates the magnitude of each token's GRPO update, while the sign and direction of the update remain anchored to the environment-derived outcome advantage. Thus, privileged information can indicate where a larger or smaller update is useful, but it does not decide whether a sampled token should be reinforced or penalized. In the original formulation, the self-distillation contribution is strongest early in training and is scheduled to decay toward vanilla GRPO, combining early dense guidance with a stable outcome-optimized training phase.

    \item \textit{SDAR}~\citep{lu2026sdar}. It keeps verifier-driven GRPO as the primary optimization backbone and adds a separately gated self-distillation objective for multi-turn agents. A teacher branch receives training-only privileged context, such as a retrieved skill, and re-scores the student's on-policy tokens; a smooth, bounded token-level gate then controls how strongly each teacher signal enters the auxiliary loss. The gate can use student uncertainty and/or the detached teacher--student log-probability gap, giving greater weight to positive teacher endorsements while softly attenuating potentially unreliable negative rejections. Unlike \textit{RLSD}, SDAR leaves the GRPO advantage itself unchanged and regulates the auxiliary distillation loss instead; the student is evaluated without privileged skill context.
\end{itemize}

For all reproduced post-training baselines, we use the same backbone and environment wrappers as OPID and match the rollout budget, task batch, number of training steps, and evaluation protocol whenever applicable. The intended differences are restricted to the optimization signal and to the explicitly stated availability of skills or other privileged training context.

\subsection{Algorithm and Extracted Skill Examples}
\label{app:algorithm_skills}
Algorithm~\ref{alg:opid} gives the full OPID training procedure, including on-policy rollout collection, hierarchical skill extraction, critical-first routing, paired scoring, and clipped policy optimization. Table~\ref{tab:hierarchical-skills-success-failure} provides representative skills extracted from successful and failed trajectories across ALFWorld, WebShop, and Search-based QA. These examples illustrate how episode-level skills capture reusable global workflows, while critical-step skills focus on sparse local decisions that influence the final outcome.

\begin{algorithm}[t]
\caption{OPID: On-Policy Skill Distillation}
\label{alg:opid}
\footnotesize
\begin{algorithmic}[1]
\Require Policy $\pi_\theta$, task set $\mathcal{Q}$, analyzer $\mathcal{A}$,
skill-injection function $H$, group size $N$, skill coefficient
$\lambda_{\mathrm{skill}}$, clipping parameter $\epsilon$, learning rate $\eta$
\For{each training iteration}
    \State $\theta_{\mathrm{old}} \gets \theta$
    \State Sample a batch of task prompts $\mathcal{B}$ from $\mathcal{Q}$
    \For{each prompt $q \in \mathcal{B}$}
        \State \textcolor{green!50!black}{\textit{// On-policy rollout group and episode advantage}}
        \State Sample $\mathcal{G}_q \gets \{\tau^{(1)},\ldots,\tau^{(N)}\}$,
        where $\tau^{(i)} \sim \pi_{\theta_{\mathrm{old}}}(\cdot\mid q)$
        \State $\mathbf{r}_q \gets \{R(\tau')\mid \tau'\in\mathcal{G}_q\}$;
        $\mu_q \gets \operatorname{mean}(\mathbf{r}_q)$;
        $\sigma_q \gets \operatorname{std}(\mathbf{r}_q)$
        \For{each trajectory $\tau \in \mathcal{G}_q$}
            \State $A^{\mathrm{ep}}_{\tau} \gets
            \bigl(R(\tau)-\mu_q\bigr)/\sigma_q$

            \State \textcolor{green!50!black}{\textit{// Hierarchical hindsight skill extraction}}
            \State $\left(s^{\mathrm{ep}}_{\tau},
            \{s^{\mathrm{step}}_{\tau,t}\}_{t\in\mathcal{C}_{\tau}}\right)
            \gets \mathcal{A}(\tau)$

            \State \textcolor{green!50!black}{\textit{// Critical-first routing and paired scoring}}
            \For{each interaction step $t$ in $\tau$}
                \State $s_{\tau,t} \gets
                \begin{cases}
                    s^{\mathrm{step}}_{\tau,t}, & t\in\mathcal{C}_{\tau},\\
                    s^{\mathrm{ep}}_{\tau}, & \text{otherwise}
                \end{cases}$
                \State $\tilde{h}_{\tau,t} \gets H(h_{\tau,t},s_{\tau,t})$
                \For{each token $\ell$ in $y_{\tau,t}$ with mask $m_{\tau,t,\ell}$}
                    \State $\ell^{\mathrm{old}}_{\tau,t,\ell} \gets
                    \log\pi_{\theta_{\mathrm{old}}}
                    (y_{\tau,t,\ell}\mid h_{\tau,t},y_{\tau,t,<\ell})$
                    \State $\ell^{\mathrm{skill}}_{\tau,t,\ell} \gets
                    \log\pi_{\theta_{\mathrm{old}}}
                    (y_{\tau,t,\ell}\mid \tilde{h}_{\tau,t},y_{\tau,t,<\ell})$
                    \State $A^{\mathrm{skill}}_{\tau,t,\ell} \gets
                    (\ell^{\mathrm{skill}}_{\tau,t,\ell}
                    -\ell^{\mathrm{old}}_{\tau,t,\ell})m_{\tau,t,\ell}$
                    \State $A^{\mathrm{ep}}_{\tau,t,\ell} \gets
                    A^{\mathrm{ep}}_{\tau}m_{\tau,t,\ell}$
                    \State $A^{\mathrm{OPID}}_{\tau,t,\ell} \gets
                    A^{\mathrm{ep}}_{\tau,t,\ell}
                    +\lambda_{\mathrm{skill}}A^{\mathrm{skill}}_{\tau,t,\ell}$
                \EndFor
            \EndFor
        \EndFor
    \EndFor

    \State \textcolor{green!50!black}{\textit{// Clipped policy optimization}}
    \State For every valid sampled token $(\tau,t,\ell)$, compute\vspace{-1pt}
    \State $\rho_{\tau,t,\ell}(\theta) \gets
    \exp\!\left(
    \log\pi_{\theta}(y_{\tau,t,\ell}\mid h_{\tau,t},y_{\tau,t,<\ell})
    -\log\pi_{\theta_{\mathrm{old}}}(y_{\tau,t,\ell}\mid h_{\tau,t},y_{\tau,t,<\ell})
    \right)$
    \State $\displaystyle
    \mathcal{L}_{\mathrm{policy}}(\theta) \gets
    -\mathbb{E}_{\tau,t,\ell}\!\left[
    \min\!\left(
    \rho_{\tau,t,\ell}(\theta)A^{\mathrm{OPID}}_{\tau,t,\ell},
    \operatorname{clip}(\rho_{\tau,t,\ell}(\theta),1-\epsilon,1+\epsilon)
    A^{\mathrm{OPID}}_{\tau,t,\ell}
    \right)\right]$
    \State $\theta \gets \theta-\eta\nabla_{\theta}
    \mathcal{L}_{\mathrm{policy}}(\theta)$
\EndFor
\end{algorithmic}
\end{algorithm}

\begin{table}[ht!]
\centering

\setlength{\belowcaptionskip}{6pt}

\caption{
\textbf{Hierarchical skills extracted from on-policy trajectories.}
For each dataset, we show one successful and one failed trajectory.
Episode-level skills summarize reusable global behavior, while critical-step
skills target sparse decision points. Step indices are 0-based analyzer keys.
}
\label{tab:hierarchical-skills-success-failure}

\scriptsize
\setlength{\tabcolsep}{2.4pt}
\renewcommand{\arraystretch}{1.10}

\begin{tabularx}{\linewidth}{
  >{\raggedright\arraybackslash}p{0.10\linewidth}
  >{\raggedright\arraybackslash}p{0.07\linewidth}
  >{\raggedright\arraybackslash}p{0.18\linewidth}
  >{\raggedright\arraybackslash}p{0.28\linewidth}
  >{\raggedright\arraybackslash}X
}
\toprule
\textbf{Dataset}
&
\textbf{Outcome}
&
\textbf{Task}
&
\textbf{Episode-level skill}
&
\textbf{Critical step skills}
\\
\midrule

\textsc{ALFWorld}
&
Success
&
clean some kettle and put it in cabinet.
&
Workflow: first locate and take the target object, then move to the cleaning
station (sinkbasin) to clean it, then go to a suitable storage location
(cabinet), open it if closed, and finally place the object inside.
&
\textbf{t=0} Go directly to the countertop or likely surface where the kettle
could be. \newline
\textbf{t=2} After acquiring the kettle, immediately go to the sinkbasin to
clean it. \newline
\textbf{t=4} After cleaning, go to a cabinet (cabinet 1) rather than pausing.
\newline
\textbf{t=6} Open the closed cabinet if needed before placing the object
inside.
\\

\addlinespace[0.30em]

\textsc{ALFWorld}
&
Failure
&
put a clean soapbar in cart.
&
Avoid placing a soapbar in the cart without first confirming it is clean.
The core mistake is ignoring the cleanliness requirement; the warning sign is
repeatedly moving the soapbar without checking or cleaning it.
&
\textbf{t=1} Take and examine the soapbar to determine if it needs cleaning.
\newline
\textbf{t=2} If the soapbar is dirty, clean it using a sink or appropriate
tool before moving to the cart.
\\

\midrule

\textsc{WebShop}
&
Success
&
Find me makeup remover for sensitive skin, nail polish with style: lagom
5 layer cotton pad, and price lower than 40.00 dollars.
&
Search broadly for the specific product name and key constraints, then click
the first matching product result to view details, verify the attributes and
price, and click 'Buy Now' to finalize.
&
\textbf{t=1} First, click the most relevant product result
(the LAGOM cotton pad) to view its detailed page.
\\

\addlinespace[0.30em]

\textsc{WebShop}
&
Failure
&
Find coffee tables with steel frame, storage space, brown, size with shelf,
and price below \$110.
&
Core mistake: Buying a product without confirming it has a steel frame.
Warning signs: product title lacks mention of 'steel frame'; search results
include many unrelated items; product page shows color and size filters but
not frame material. Avoid relying on partial matches; always verify all
specific attributes, especially material, before finalizing purchase.
&
\textbf{t=1} Before clicking on a product, examine its title and check if it
explicitly mentions steel frame or other required attributes. \newline
\textbf{t=2} On the product page, click on 'Description' or 'Features' to
verify the steel frame and shelf size before clicking 'Buy Now'.
\\

\midrule

\textsc{Search}
&
Success
&
Who illustrated Hunter S. Thompson's novel Fear and Loathing in Las Vegas?
&
Workflow: First, query using core entities (author/title) to gather context;
if initial search lacks direct answer, reformulate query specifically targeting
the required attribute (illustrator) and use the new results to extract the
answer.
&
\textbf{t=1} If the initial search results do not directly answer the question,
reformulate the search query to specifically target the missing attribute
(here, 'illustrated by').
\\

\addlinespace[0.30em]

\textsc{Search}
&
Failure
&
What is the full founding date of GroenLinks, the party led by Jesse Feras
Klaver?
&
Avoid ignoring crucial temporal precision; when the task demands a specific
date, search or extract the full date, not just the year, even if the year is
initially prominent. Warning sign: Answering with only a year when documents
contain more precise information.
&
\textbf{t=2} Extract the full founding date from documents about GroenLinks,
not just the year.
\\

\bottomrule
\end{tabularx}
\end{table}

\subsection{Implementation Details}\label{app:implementation_details}
\paragraph{Metrics.}
For ALFWorld, we compute the success rate for each task type and report their macro-average:
\begin{equation}
    \mathrm{ALFWorld\text{-}Avg}
    = \frac{1}{6}\sum_{c=1}^{6}\mathrm{SR}_{c}.
\end{equation}
For Search-based QA, we compute answer accuracy separately on each of the seven datasets and report the unweighted macro-average:
\begin{equation}
    \mathrm{Search\text{-}Avg}
    = \frac{1}{7}\sum_{d=1}^{7}\mathrm{Acc}_{d}.
\end{equation}
For WebShop, the reported \textit{Score} is the mean normalized task score returned by the environment, multiplied by 100, and \textit{Succ.} is the percentage of tasks with exact success. 

\paragraph{Trajectory analyzer.}
After each on-policy episode terminates, we serialize the task prompt, step-indexed observations, policy responses/actions, environment feedback, and terminal outcome into an ordered trajectory record. An LLM-based analyzer then maps this record to one episode-level skill and a sparse set of critical-step skills. Step indices are zero-based, consistent with Table~\ref{tab:hierarchical-skills-success-failure}. By default, we use GLM-5.2~\citep{zai2026glm52} as the analyzer, with temperature set to 0.4 and maximum output length set to 4096. We limit the max number of identified critical steps at 5 for ALFWorld and WebShop, and at 2 for Search-based QA.

\paragraph{Backbones and training schedule.}
We use Qwen2.5-3B-Instruct and Qwen2.5-7B-Instruct~\citep{yang2024qwen25}, as well as Qwen3-1.7B-Instruct~\citep{yang2025qwen3}. All models are trained for 150 update steps. The training batch size reported in the main paper is 16 for ALFWorld and WebShop and 128 for Search-based QA. Table~\ref{tab:rl_hyperparameters} records the remaining hyperparameters that are required for exact reproduction.
\begin{table*}[t]
\centering
\small
\setlength{\tabcolsep}{6pt}
\renewcommand{\arraystretch}{1.15}
\caption{RL training hyperparameters.}
\label{tab:rl_hyperparameters}
\begin{tabular}{@{}p{0.4\textwidth}p{0.5\textwidth}@{}}
\toprule
\textbf{Hyperparameter} & \textbf{Value} \\
\midrule

Training steps
& 150  \\

Training batch size
& 16 for ALFWorld and WebShop; 128 for Search \\

Rollout group size $N$
& 8 \\

Learning rate
& $1\times10^{-6}$  \\

PPO clip parameter $\epsilon$
& 0.2 \\

Skill coefficient $\lambda_{\mathrm{skill}}$
& 0.001 \\

KL regularization coefficient
& 0.01 \\

Maximum prompt length
& 2,048 for ALFWorld ; 4,096 for WebShop and Search  \\

Response lengths
& 512 \\

Maximum interaction steps
& 30 for ALFWorld, 15 for WebShop, and 4 for Search. \\
\bottomrule
\end{tabular}
\end{table*}

\paragraph{Computing details.}
Training is conducted on 8 Nvidia A800 80G GPUs.

\section{Supplementary Results}\label{app:supp_results}

\subsection{Detailed Sample Efficiency Comparison}
\label{app:sample_efficiency}
Table~\ref{tab:sample_efficiency} reports the ALFWorld success rate when only a fraction of the training data is used. OPID consistently improves over GRPO across all data budgets. The gains are especially large in the low- and mid-data regimes, reaching +15.6 points with 60\% of the data and +20.3 points with 80\% of the data. These results suggest that trajectory-derived hindsight skills allow OPID to extract more supervision from each rollout, making outcome-based RL less dependent on large numbers of environment interactions.

\begin{table}[htbp]
\centering
\caption{\textbf{Sample efficiency comparison on ALFWorld.} We report success rates under different fractions of the training data. The $\Delta$ row shows the absolute improvement of OPID over GRPO, indicating that OPID provides stronger gains especially in low- and mid-data regimes.}
\vspace{0.05in}
\label{tab:sample_efficiency}
\setlength{\tabcolsep}{8pt}
\renewcommand{\arraystretch}{1.08}
\begin{tabular}{cccccc}
\toprule
\textbf{Method} & \textbf{20\%} & \textbf{40\%} & \textbf{60\%} & \textbf{80\%} & \textbf{100\%} \\
\midrule
GRPO & 27.3 & 42.2 & 56.3 & 58.6 & 75.0 \\
\rowcolor{gray!10}
OPID & 36.7 & 54.7 & 71.9 & 78.9 & 84.3 \\
\rowcolor{blue!6}
$\Delta$ & +9.4 & +12.5 & +15.6 & +20.3 & +9.3 \\
\bottomrule
\end{tabular}
\end{table}


\subsection{Cross-Domain Generalization}
\label{app:ood_generalization}
Table~\ref{tab:ood_generalization_alfworld} evaluates transfer to the ALFWorld unseen split. OPID improves the average success rate over GRPO by +7.7 points, with particularly clear gains on \textit{Look} and \textit{Heat}. This indicates that OPID does not merely fit the observed training trajectories. Instead, the distilled episode-level workflows and step-level decision rules retain value under unseen environment configurations.

\begin{table}[ht!]
\centering
\caption{\textbf{Cross-domain generalization results on ALFWorld Unseen.} We report success rates across six unseen task types and their average. OPID improves the average success rate over GRPO, indicating that trajectory-derived skill supervision transfers beyond the training environments.}
\vspace{0.05in}
\label{tab:ood_generalization_alfworld}
\small
\setlength{\tabcolsep}{9pt}
\renewcommand{\arraystretch}{1.08}
\begin{tabular}{cccccccc}
\toprule[1.2pt]
& \multicolumn{7}{c}{\textbf{ALFWorld Unseen}} \\
\cmidrule(lr){2-8}
\textbf{Method} & \textbf{Pick} & \textbf{Look} & \textbf{Clean} & \textbf{Heat} & \textbf{Cool} & \textbf{Pick2} & \textbf{Avg.} \\
\midrule
ReAct        & 17.4 & 6.7  & 8.8  & 7.4  & 9.1  & 0.0  & 8.2  \\
GRPO         & 73.9 & 60.0 & \textbf{82.4} & 59.3 & 72.7 & \textbf{76.9} & 70.9 \\
\rowcolor{gray!10}
OPID         & \textbf{78.3} & \textbf{86.7} & \textbf{82.4} & \textbf{77.8} & \textbf{77.3} & 69.2 & \textbf{78.6} \\
\rowcolor{blue!6}
$\Delta$     & +4.4 & +26.7 & +0.0 & +18.5 & +4.6 & -7.7 & +7.7 \\
\bottomrule[1.2pt]
\end{tabular}
\end{table}

\subsection{Training Diagnostics and Skill Extraction Patterns}
\label{app:training_diagnostics}
Figures~\ref{fig:critical_steps_alfworld}--\ref{fig:analyser_prompt} provide additional diagnostics for the OPID training pipeline. Figure~\ref{fig:critical_steps_alfworld} reports the average number of critical steps identified on ALFWorld, illustrating that OPID applies step-level supervision sparsely rather than assigning local skills to every decision. Figure~\ref{fig:training_advantage_alfworld} further visualizes the training advantage dynamics, complementing the main-paper training curves and showing how OPID reshapes the learning signal during policy optimization. Figure~\ref{fig:analyser_prompt} shows the analyzer prompt used to convert completed trajectories into hierarchical skills.

\begin{figure}[t]
    \centering
    \includegraphics[width=\textwidth]{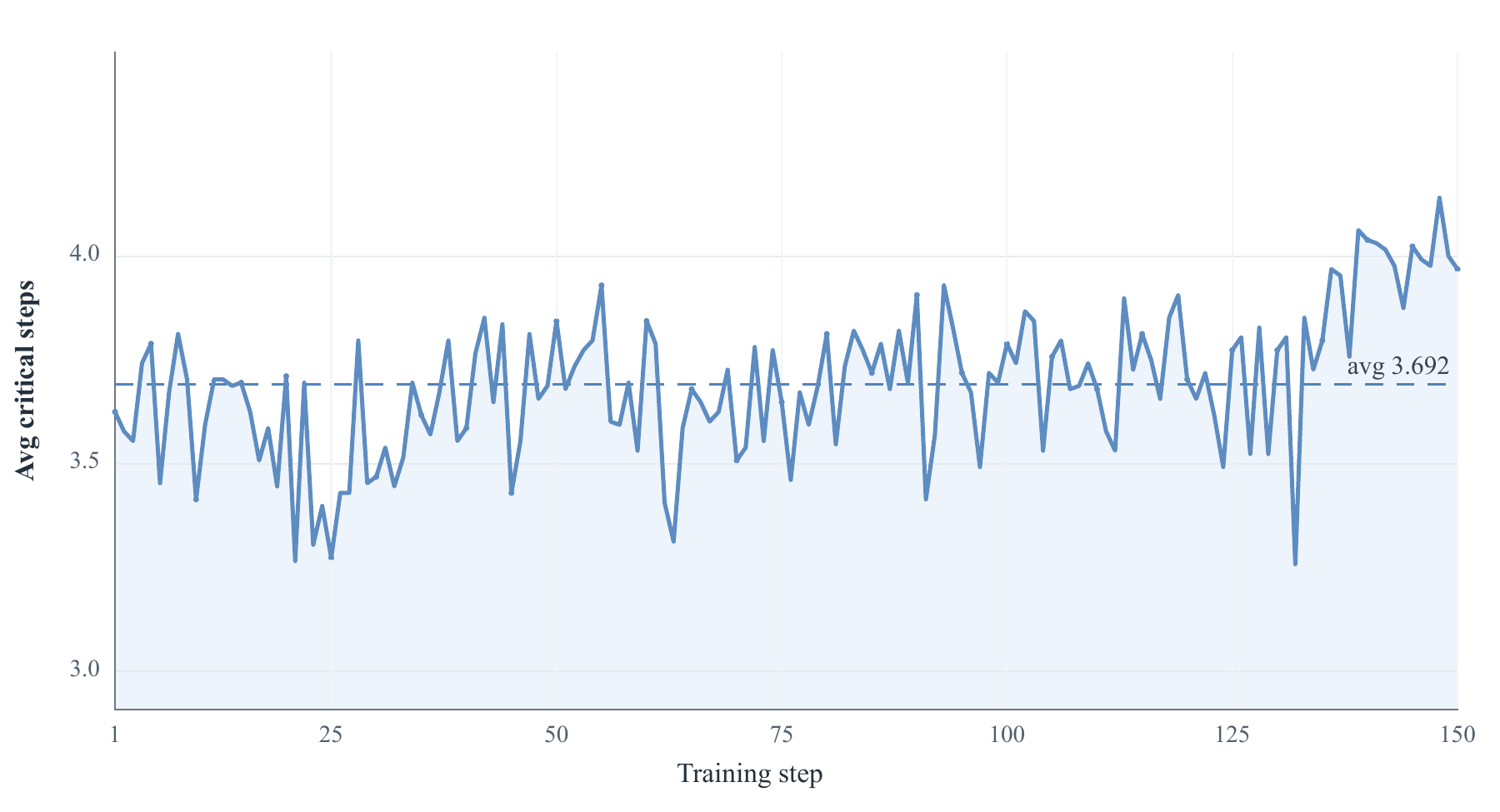}
    \caption{\textbf{Average critical steps per sequence on ALFWorld.} The curve reports how many timesteps are selected by the analyzer for step-level hindsight skills in each trajectory. The relatively small number of critical steps indicates that OPID applies local skill supervision selectively, while relying on episode-level skills as default guidance for non-critical decisions.}
    \label{fig:critical_steps_alfworld}
\end{figure}

\begin{figure}[t]
    \centering
    \includegraphics[width=\textwidth]{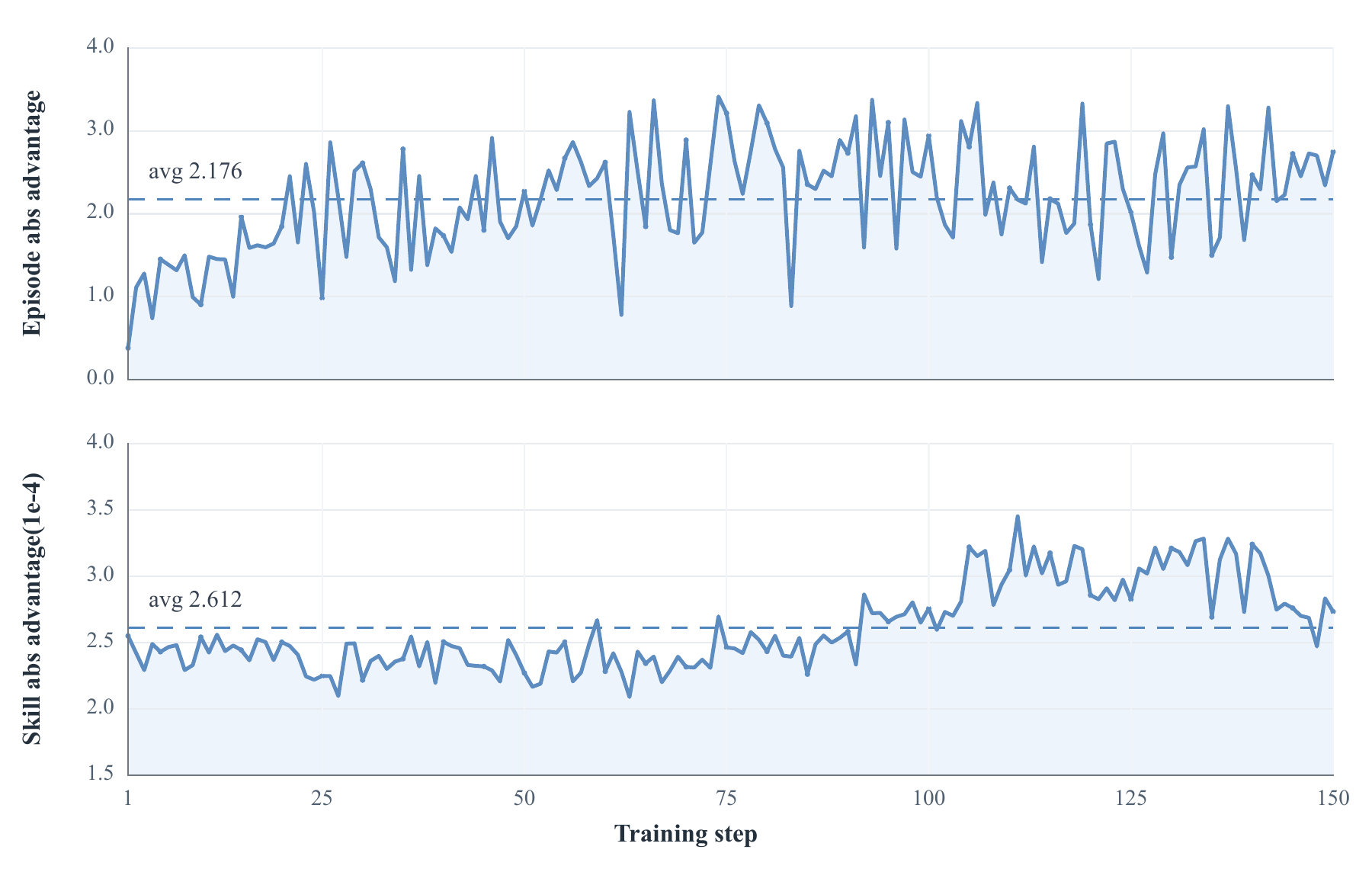}
    \caption{\textbf{Magnitudes of episode-level and skill-guided advantage signals during OPID training.} Episode abs advantage measures the mean absolute advantage from group-relative outcome rewards, while skill abs advantage measures the mean absolute advantage induced by skill-guided log-probability shifts. The comparison shows how OPID combines sparse trajectory-level feedback with dense skill-conditioned supervision throughout optimization.}
    \label{fig:training_advantage_alfworld}
\end{figure}

\section{Case Study}\label{app:case_study}
Figures~\ref{fig:alfworld_case1}--\ref{fig:webshop_case2} provide illustrative examples from the ALFWorld, Search-QA, and WebShop benchmarks.

\section{Additional Discussion}\label{app:discussion}
OPID studies how completed on-policy trajectories can be reused as hindsight supervision for long-horizon agentic reinforcement learning. A natural next step is to evaluate this idea in broader interactive environments where agents must discover latent rules, maintain long-term state, and adapt through extended interaction. Benchmarks such as OdysseyArena~\citep{xu2026odysseyarena}, AgentBench~\citep{liu2023agentbench}, WebArena~\citep{zhou2023webarena}, Mind2Web~\citep{deng2023mind2web}, and VisualWebArena~\citep{koh2024visualwebarena} provide complementary stress tests beyond the embodied, shopping, and search-based settings considered in this paper. These environments would test whether trajectory-derived hindsight skills remain useful when the agent must handle longer horizons, richer interfaces, and more open-ended forms of exploration.

Another direction is to enrich the structure of hindsight skills. OPID currently extracts episode-level and step-level skills from completed trajectories and routes them according to decision criticality. Future work could combine this on-policy extraction with higher-level reasoning abstractions, such as search-discovered reasoning patterns or reusable thought structures~\citep{wu2024beyond}, and with policy-aware exploration mechanisms developed for long-horizon agent learning~\citep{wu2026spark,lu2026sdar}. Such extensions may allow agents to aggregate skills across trajectories, identify recurring failure modes, and form more compositional behavioral rules while preserving OPID's key design choice: skills are used to shape training, not retrieved as privileged context at inference time.

Finally, OPID opens several deployment-oriented directions. Since the analyzer and skill-conditioned scoring are used only during training, the learned policy incurs no additional inference-time skill retrieval cost. Nevertheless, the training pipeline can still benefit from more efficient inference and scoring mechanisms. Speculative and retrieval-parallel decoding methods such as \textsc{Double}~\citep{shen2026double} may reduce the cost of repeated model scoring during skill-conditioned distillation. In parallel, extending OPID to more perceptual and embodied settings, including active embodied intelligence benchmarks such as RobotEQ~\citep{fang2026roboteq}, could test whether hindsight skill supervision helps agents acquire not only task completion strategies, but also socially and spatially grounded decision rules.

\begin{figure}[htbp]
    \centering
    \includegraphics[width=\linewidth]{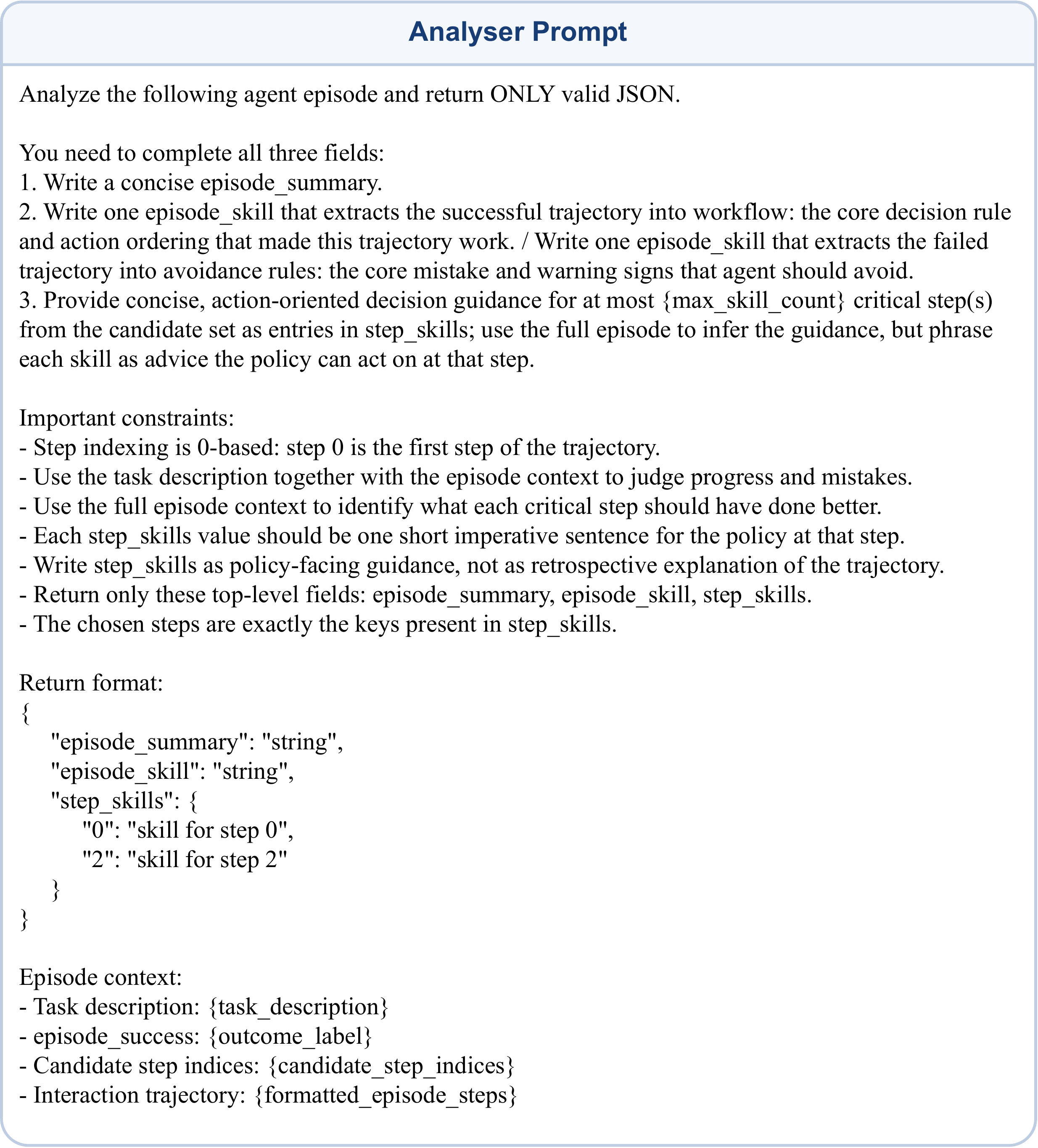}
    \caption{
        Prompt of analyzer.
    }
    \label{fig:analyser_prompt}
\end{figure}

\begin{figure}[hp]
    \centering
    \includegraphics[width=\textwidth]{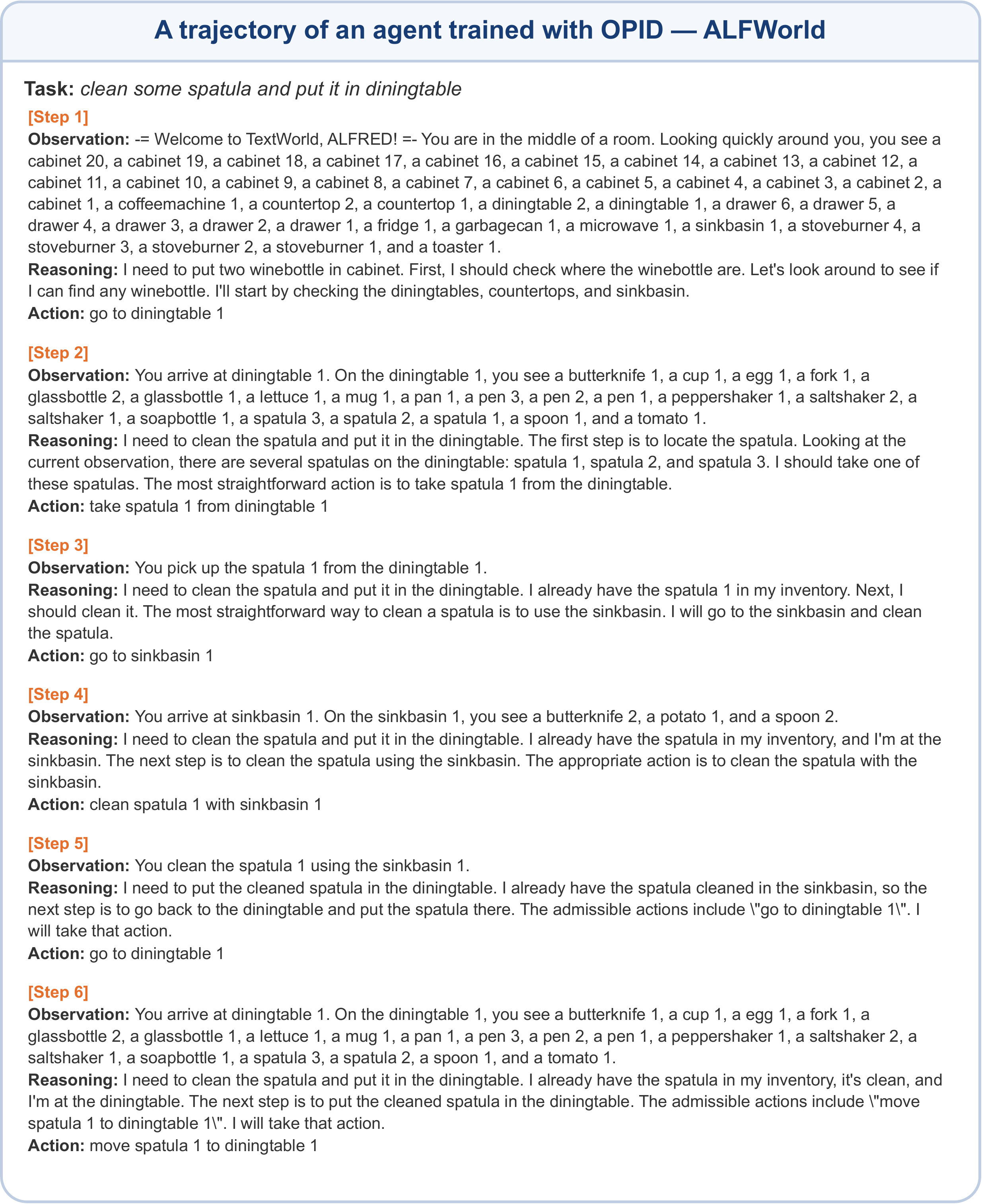}
    \caption{A full trajectory of OPID on ALFWorld Example 1.}
    \label{fig:alfworld_case1}
\end{figure}
\begin{figure}[hp]
    \centering
    \includegraphics[width=\textwidth]{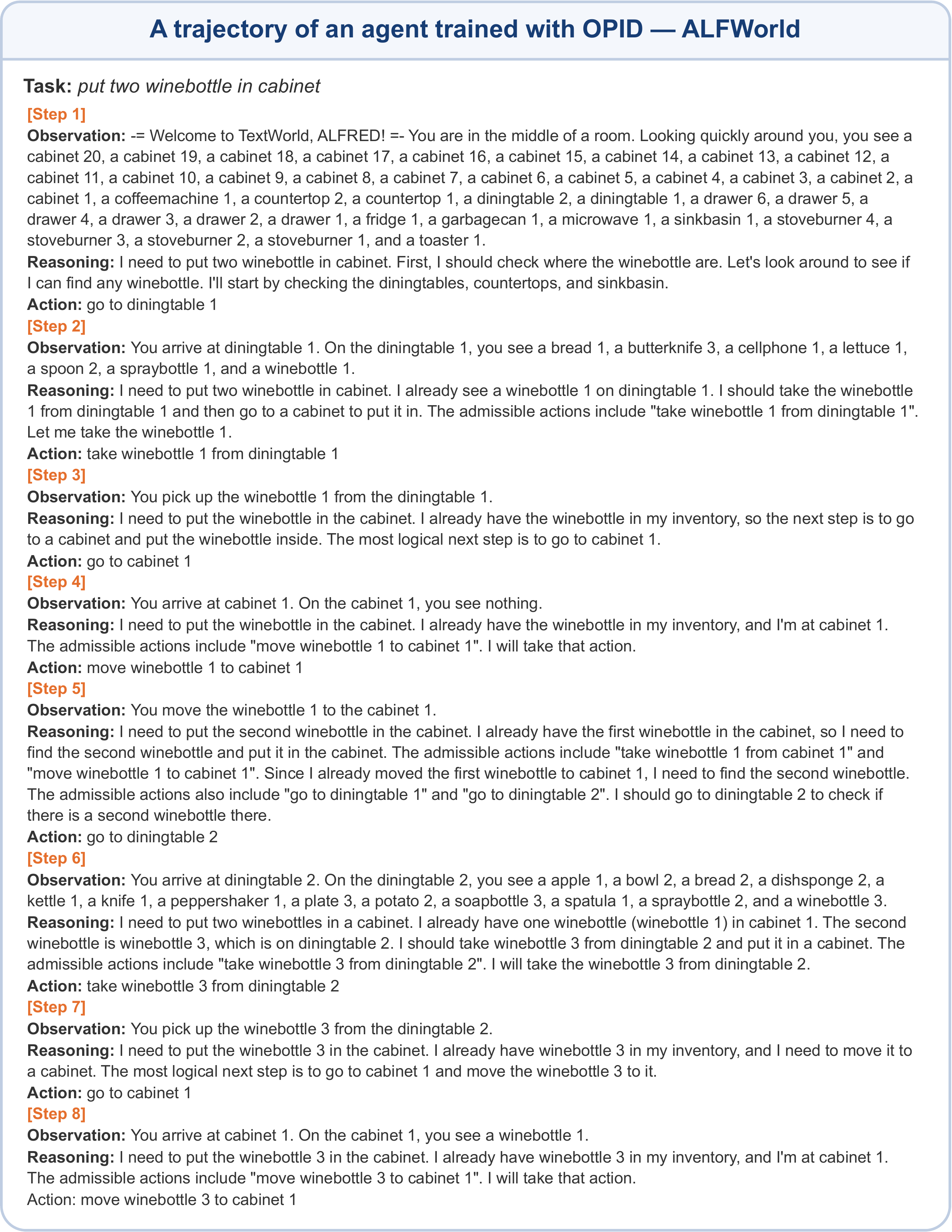}
    \caption{A full trajectory of OPID on ALFWorld Example 2.}
    \label{fig:alfworld_case2}
\end{figure}
\begin{figure}[hp]
    \centering
    \includegraphics[width=\textwidth]{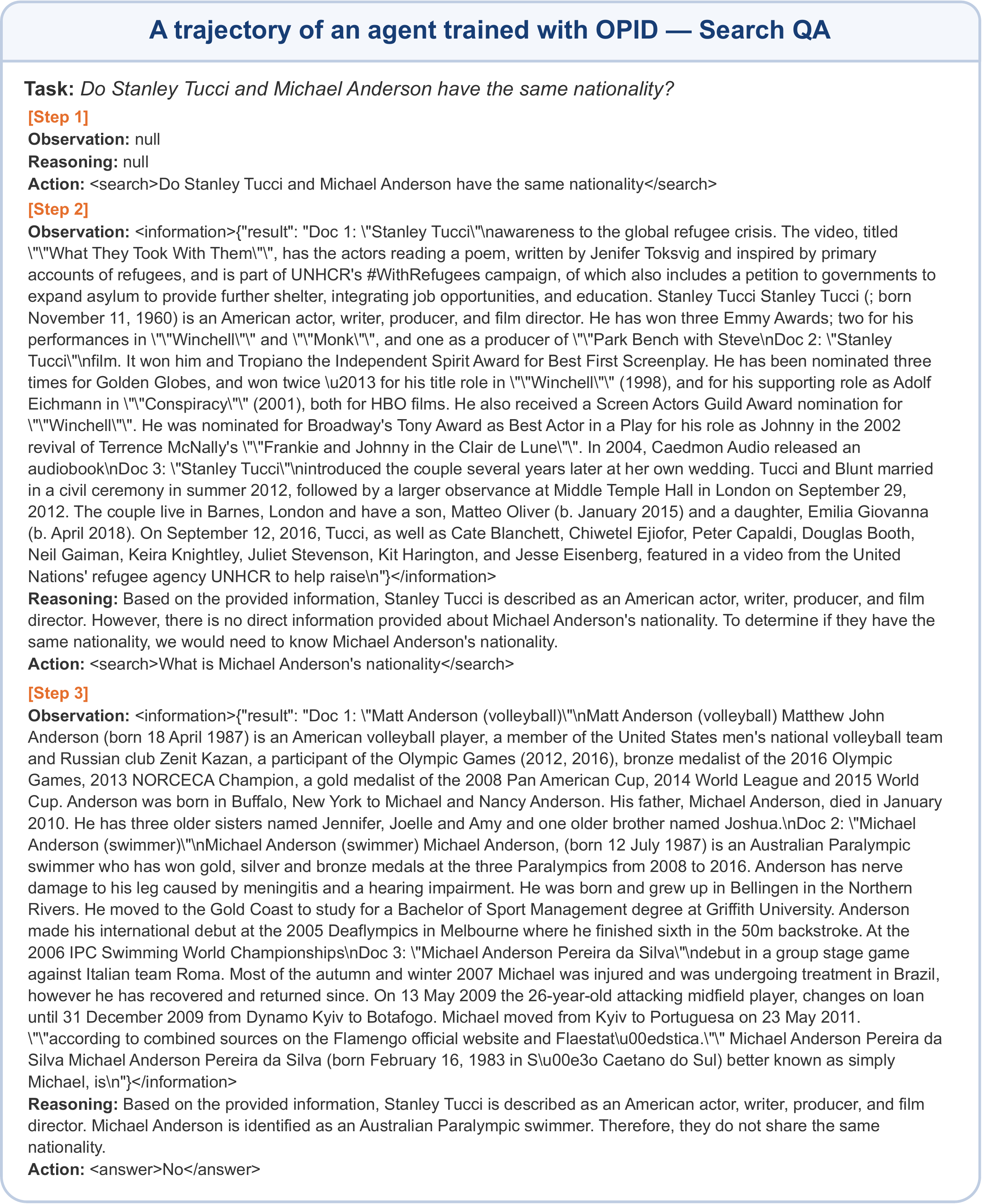}
    \caption{A full trajectory of OPID on Search-QA Example 1.}
    \label{fig:searchqa_case1}
\end{figure}
\begin{figure}[hp]
    \centering
    \includegraphics[width=\textwidth]{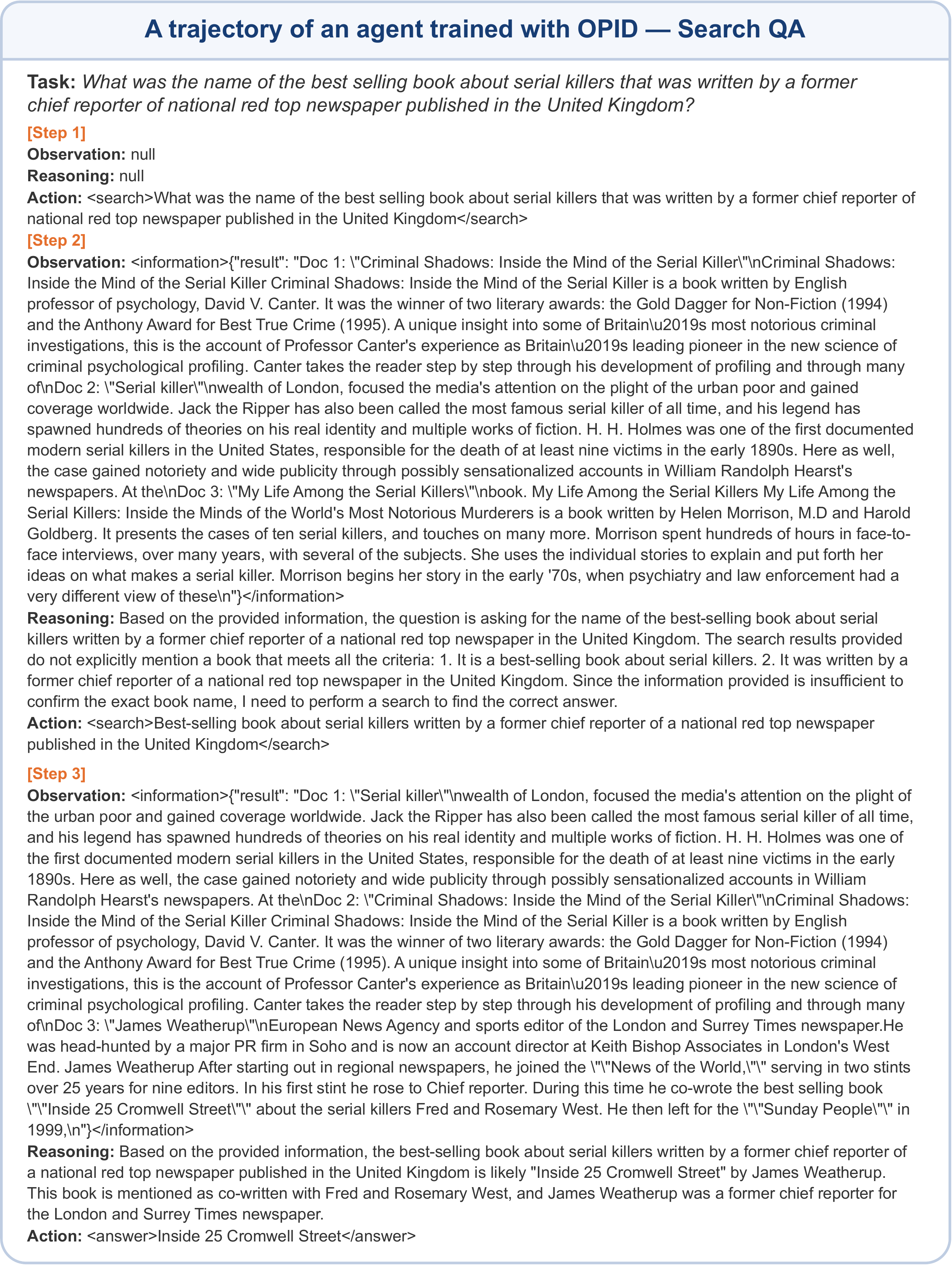}
    \caption{A full trajectory of OPID on Search-QA Example 2.}
    \label{fig:searchqa_case2}
\end{figure}
\begin{figure}[hp]
    \centering
    \includegraphics[width=\textwidth]{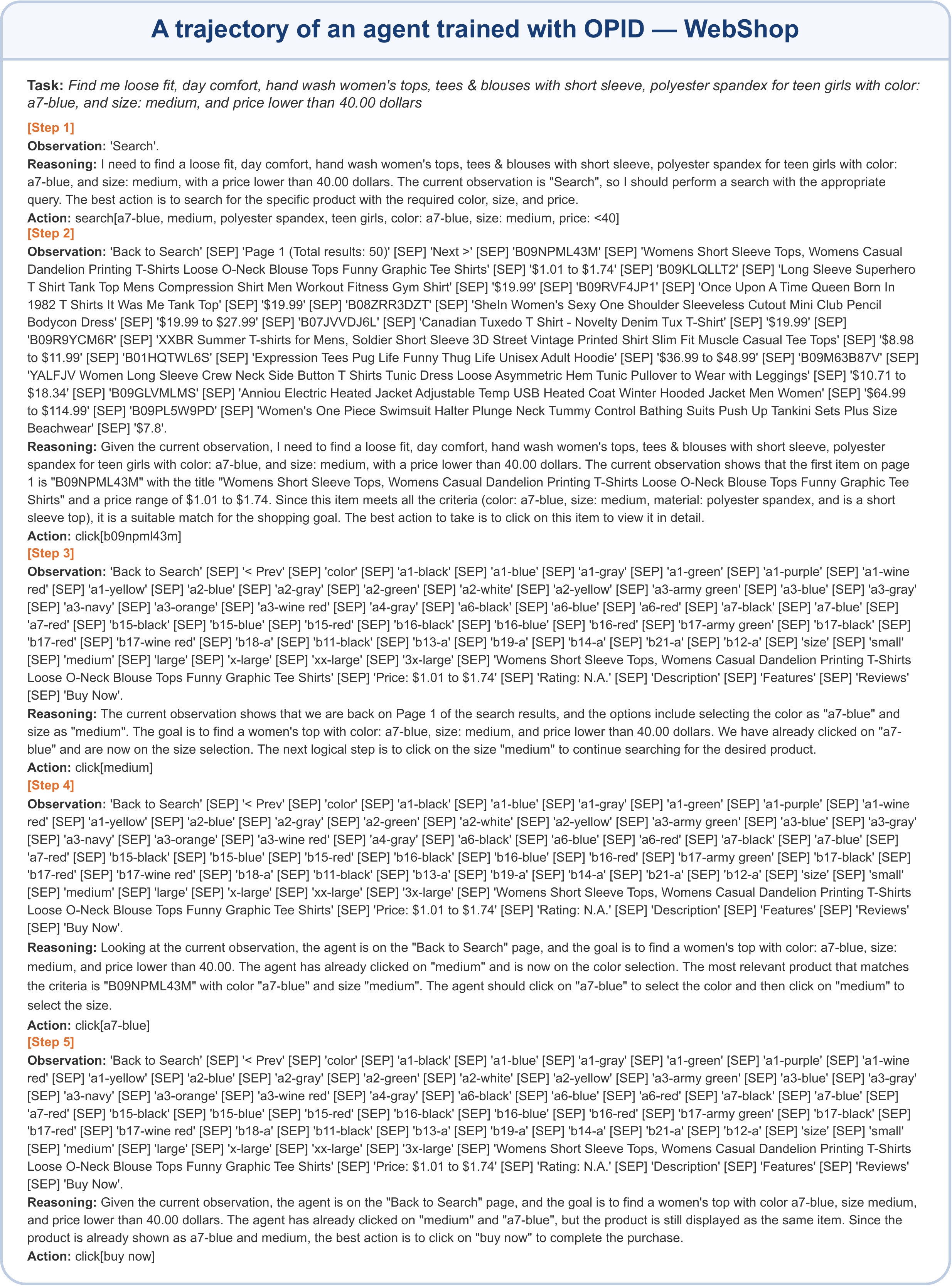}
    \caption{A full trajectory of OPID on Webshop Example 1.}
    \label{fig:webshop_case1}
\end{figure}
\begin{figure}[hp]
    \centering
    \includegraphics[width=\textwidth]{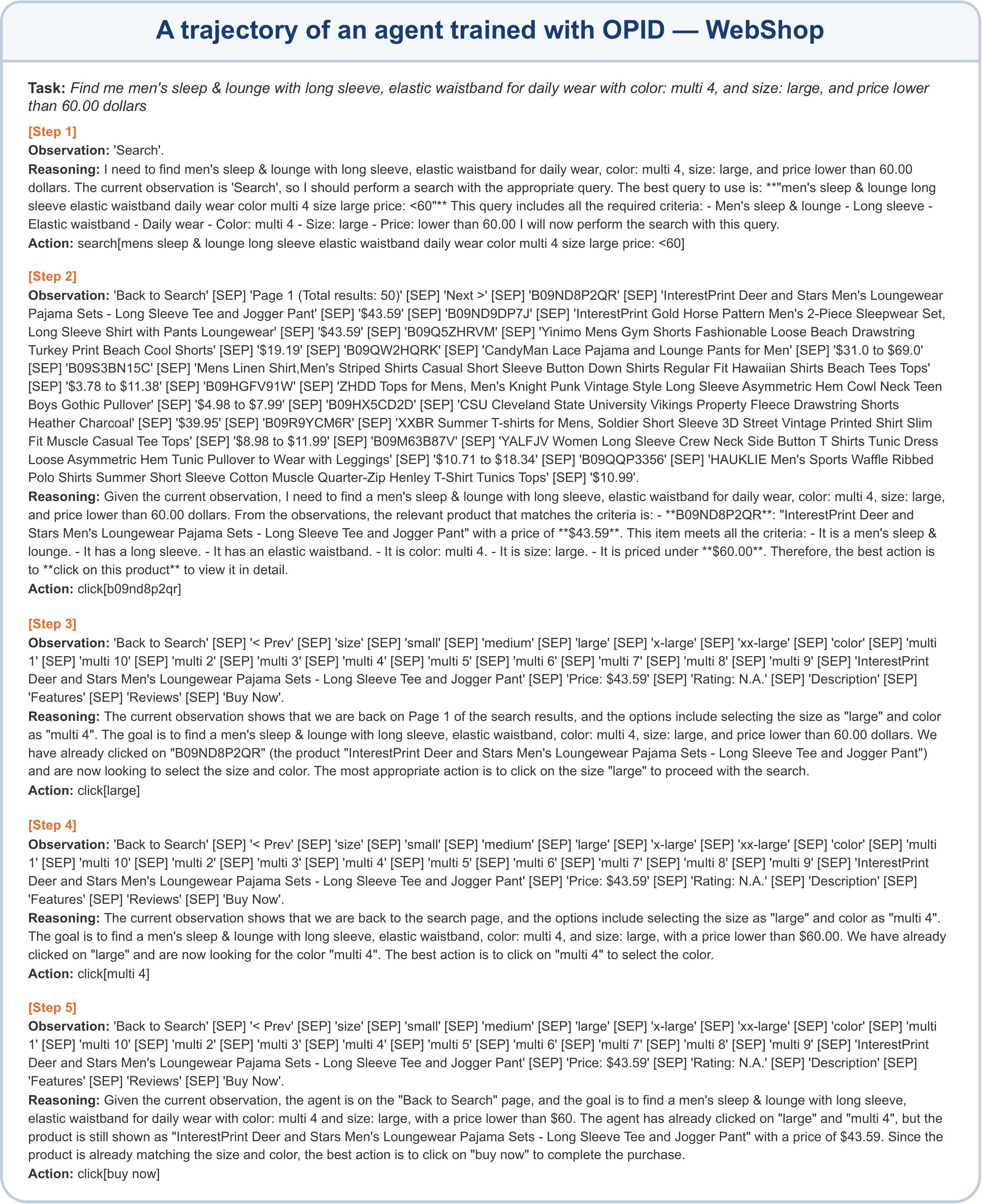}
    \caption{A full trajectory of OPID on Webshop Example 2.}
    \label{fig:webshop_case2}
\end{figure}

\end{document}